\documentclass[nohyperref]{article}

\usepackage{microtype}
\usepackage{graphicx}
\usepackage{subfigure}
\usepackage{booktabs} 

\usepackage{hyperref}

\usepackage[accepted]{icml2022}

\usepackage{amsmath}
\usepackage{amssymb}
\usepackage{mathtools}
\usepackage{amsthm}
\usepackage{commath}
\usepackage{dsfont}
\usepackage{tcolorbox}
\usepackage{enumitem}
\usepackage{minibox}
\usepackage{framed}
\usepackage{xspace}
\usepackage{cancel}

\setitemize{itemsep=0pt,topsep=0pt,parsep=0pt,partopsep=0pt}

\usepackage[capitalize,noabbrev]{cleveref}

\theoremstyle{plain}
\newtheorem{theorem}{Theorem}[section]
\newtheorem{proposition}[theorem]{Proposition}
\newtheorem{lemma}[theorem]{Lemma}

\theoremstyle{definition}
\newtheorem{definition}[theorem]{Definition}

\newtheorem{remark}[theorem]{Remark}
\newtheorem{claim}[theorem]{Claim}
\newtheorem{fact}[theorem]{Fact}
\newtheorem{example}[theorem]{Example}

\newcommand{\diam}{\mathrm{diam}}

\newcommand{\CERM}{\mathrm{C\text{-}ERM}}
\newcommand{\one}{\mathds{1}}
\newcommand{\sign}{\mathrm{sign}}
\newcommand{\N}{\mathrm{N}}

\newcommand{\XTD}{\mathrm{XTD}}
\newcommand{\err}{\mathrm{err}}
\newcommand{\poly}{\mathrm{poly}}
\newcommand{\NN}{\mathbb{N}}
\newcommand{\spn}{\mathrm{span}}

\newcommand{\VC}{\text{VC}}
\newcommand{\DIS}{\mathrm{DIS}}
\newcommand{\CC}{\mathrm{Cost}}
\newcommand{\TV}{\mathrm{TV}}

\DeclareMathOperator*\argmin{argmin}
\DeclareMathOperator*\argmax{argmax}
\newcommand{\KL}{\mathrm{KL}}
\newcommand{\kl}{\mathrm{kl}}
\newcommand{\Ccal}{\mathcal{C}}

\newcommand{\Ocal}{\mathcal{O}}
\newcommand{\Xcal}{\mathcal{X}}
\newcommand{\Ycal}{\mathcal{Y}}
\newcommand{\Hcal}{\mathcal{H}}

\newcommand{\Fcal}{\mathcal{F}}
\newcommand{\Zcal}{\mathcal{Z}}

\newcommand{\Tcal}{\mathcal{T}}

\newcommand{\Acal}{\mathcal{A}}

\newcommand{\Wcal}{\mathcal{W}}

\newcommand{\EE}{\mathbb{E}}

\newcommand{\Exp}{\mathrm{Exponential}}

\newcommand{\polylog}{\mathrm{polylog}}
\newcommand{\PP}{\mathbb{P}}
\newcommand{\QQ}{\mathbb{Q}}
\newcommand{\RR}{\mathbb{R}}
\newcommand{\inner}[2]{\langle #1,#2 \rangle}
\newcommand{\ind}{\mathds{1}}

\newcommand{\Uniform}{\mathrm{Uniform}}

\newcommand{\RETURN}{\STATE \textbf{return }}
\renewcommand\Pr{\mathop{\rm Pr}\nolimits}
\newcommand{\pminor}{\underline{p}}

\newcommand{\thres}{\tau}
\newcommand{\bs}{\ensuremath{\textsc{Binary-Search}}\xspace}
\newcommand{\estimg}{\ensuremath{\textsc{Estimate-Positive}}\xspace}
\newcommand{\hs}{\text{lin}}
\newcommand{\opt}{\mathrm{OPT}}
\newcommand{\setc}{\mathsf{SC}}
\newcommand{\mc}{\mathsf{MC}}
\newcommand{\mcsoln}{L}
\newcommand{\oli}{N}

\renewcommand{\algorithmiccomment}[1]{\textcolor{blue}{//#1}\hfill}
\usepackage[textsize=tiny]{todonotes}

\icmltitlerunning{Active Fairness Auditing}

\begin{document}

\twocolumn[
\icmltitle{Active Fairness Auditing}




\icmlsetsymbol{equal}{*}

\begin{icmlauthorlist}
\icmlauthor{Tom Yan}{equal,cmu}
\icmlauthor{Chicheng Zhang}{equal,uofa}
\end{icmlauthorlist}

\icmlaffiliation{cmu}{Carnegie Mellon University}
\icmlaffiliation{uofa}{University of Arizona}

\icmlcorrespondingauthor{Tom Yan}{tyan2@andrew.cmu.edu}
\icmlcorrespondingauthor{Chicheng Zhang}{chichengz@cs.arizona.edu}

\icmlkeywords{Machine Learning, ICML}

\vskip 0.3in
]



\printAffiliationsAndNotice{\icmlEqualContribution} 

\begin{abstract}
The fast spreading adoption of machine learning (ML) by companies across industries poses significant regulatory challenges. One such challenge is scalability: how can regulatory bodies efficiently \emph{audit} these ML models, ensuring that they are fair? In this paper, we initiate the study of query-based auditing algorithms that can estimate the demographic parity of ML models in a query-efficient manner. We propose an optimal deterministic algorithm, as well as a practical randomized, oracle-efficient algorithm with comparable guarantees. Furthermore, we make inroads into understanding the optimal query complexity of randomized active fairness estimation algorithms. Our first exploration of active fairness estimation aims to put AI governance on firmer theoretical foundations.
\end{abstract}




\section{Introduction}
\label{sec:intro}

With growing usage of artificial intelligence (AI) across industries, governance efforts are increasingly ramping up. A key challenge in these regulatory efforts is the problem of scalability. Even for well-resourced countries like Norway, which is pioneering efforts in AI governance, regulators are only able to monitor and engage with a ``small fraction of the companies''~\cite{mccarthy21}. This growing issue calls for a better understanding of \emph{efficient} approaches to auditing machine learning (ML) models, which we now formalize.

\textbf{Problem Formulation:} A regulatory institution is interested in auditing a model $h^*: \Xcal \to \cbr{-1,1}$ held by a company (e.g. a lending company in the finance sector), where $\Xcal$ is the feature space (e.g. of all information supplied by users). 
We assume that the regulatory institution only has knowledge of the hypothesis class $\Hcal$ where $h^*$ comes from (e.g. the family of linear classifiers), and it would like to estimate $\mu(h^*)$ for a function $\mu$ that measures the model property of interest. 
To this end, the institution is allowed to send black-box queries to the model $h^*$, i.e. send the company a query example $x$ and receive $h^*(x)$. The regulatory institution's goal is to \emph{efficiently} estimate $\mu(h^*)$ to within an error of at most $\epsilon > 0$.

We measure an algorithm's \emph{efficiency} in terms of both its \emph{query complexity} and \emph{computational complexity}. Having an auditing algorithm with low query and computational complexity naturally helps to address the scalability challenge: greater efficiency means that each audit may be processed faster and more audits may be processed at a time.

\textbf{Property of Interest:} While which properties $\mu$ to assess is still heavily debated by regulators, we initiate the study of auditing algorithms by focusing on fairness, a mainstay in regulatory focuses. In particular, the $\mu$ we will consider will be Demographic Parity (DP)\footnote{While fairness is the focus of our work, our algorithm may be adapted to any $\mu$ which is a function of $\Xcal$ and $h^*$.}: given distribution $D_X$ over $\Xcal \times \cbr{0,1}$ (where feature $x$ and sensitive attribute $x_A$ are jointly drawn from), $\mu_{D_X}(h) = \Pr_{(x,x_A) \sim D_X}(h(x) = 1| x_A = 1) - \Pr_{(x,x_A) \sim D_X}(h(x) = 1| x_A = 0)$. For brevity, when it is clear from context, we abbreviate $\Pr_{D_X}, \mu_{D_X}$ as $\Pr, \mu$, respectively. DP measures the degree of disparate treatment of model $h$ on the two sub-populations 
$x \mid x_A = 0$ and $x \mid x_A = 1$, which we assume are non-negligible: $\pminor := \min(\Pr(x_A = 1), \Pr(x_A = 0)) = \Omega(1)$. Achieving a small Demographic Parity may be thought of as a stronger version of the US Equal Employment Opportunity Commission’s “four-fifths rule”.\footnote{
The ``selection rate for any race, sex, or ethnic group [must be at least] four-fifths (4/5) (or eighty percent) of the rate for the group with the highest rate.''}

To focus on query complexity, we will abstract away the difficulty of evaluating $\mu$ by assuming that $D_X$ is known, and thus for any $h$, we may evaluate $\mu(h)$ to arbitrary precision; for instance, this may be achieved with the availability of an arbitrarily large number of (unlabeled) samples randomly drawn from $x \mid x_A = 0$ and $x \mid x_A = 1$. Our main challenge is that we do not know $h^*$ and only want to query $h^*$ \emph{insofar as to be able to accurately estimate $\mu(h^*)$}.

\textbf{Guarantees of the Audit:} In our paper, we investigate algorithms that can provide two types of guarantees. The first is the natural, \emph{direct estimation accuracy}: the estimate returned by the algorithm should be within $\epsilon$ of $\mu(h^*)$. 

The second is that of \emph{manipulation-proof} (MP) estimation. Audits can be very consequential to companies as they may be subject to hefty penalties if caught with violations. Not surprisingly, there have been effortful attempts in the past to avoid being caught with violations~\citep[e.g.][]{hotten15} by ``gaming'' the audit. We formulate our notion of manipulation-proofness in light of one way the audit may be gamed, which we now describe. Note that all the auditor knows about the model used by the company is that it is consistent with the queried labels in the audit. So, while our algorithm may have estimated $\mu(h^*)$ accurately during audit-time, nothing stops the company from changing its model \emph{post-audit} from $h^*$ to a different model $h_{\text{new}} \in \Hcal$ (e.g to improve profit), so long as $h_{\text{new}}$ is still consistent with the queries seen during the audit. With this, we also look to understand: given this post-hoc possibility of manipulation, can we devise an algorithm that nonetheless ensures the algorithm's estimate is within $\epsilon$ of $\mu(h_{\text{new}})$? 

Indeed, a robust set of audit queries would serve as a \emph{certificate} that no matter which model the company changes to after the audit, its $\mu$-estimation would remain accurate. Given a set of classifiers $V$, a classifier $h$, and a unlabeled dataset $S$, define the version space~\cite{mitchell1982generalization} induced by $S$ to be $V( h, S ) := \cbr{h' \in V: h'(S) = h(S)}$. An auditing algorithm is $\epsilon$-manipulation-proof if, for any $h^*$, it outputs a set of queries $S$ and estimate $\hat{\mu}$ that guarantees that $\max_{h \in \Hcal(h^*, S)} \abs{ \mu(h) - \hat{\mu} } \leq \epsilon$.


\textbf{Baseline: i.i.d Sampling:} One natural baseline that comes to mind for the direct estimation is i.i.d sampling. We sample $O(1 / \epsilon^2)$ examples i.i.d from the distribution $x \mid x_A = i$ for $i \in \{0, 1\}$, query $h^*$ on these examples and take the average to obtain an estimate of $\Pr(h^*(x) = +1 \mid x_A = i)$. Finally, we take the difference of these two estimates as our final DP estimate. 
By Hoeffding's Inequality, with high probability, this estimate is $\epsilon$-accurate, and this estimation procedure makes $O(1 / \epsilon^2)$ queries. 

However, i.i.d sampling is not necessarily MP. To see an example, let there be $2n$ points in group $x_A = 1$ with $n = 1 / \epsilon^2$ that are shattered by $\Hcal$ and $D_X$ is uniform over these points. Suppose that all points in group $x_A = 0$ are labeled the same: $\Pr_{D_X}(h(x) = 1| x_A = 0) = 0, \forall h \in \Hcal$. Then, $\mu$-estimation reduces to estimating the proportion of positives in group $x_A = 1$. i.i.d sampling will randomly choose $n$ of these data points to see, and it will produce an $\epsilon$-accurate estimate of $\mu(h^*)$. However, we do not see the other $n$ points. Since the $2n$ points are shattered by $\Hcal$, \emph{after} the queried points are determined, we see that the company can increase or decrease DP by up to $1/2$ by switching to a different model.

To obtain both direct and MP estimation, it seems promising then to examine algorithms that make use of \emph{non-iid} sampling. Moreover, for MP, we observe that the auditing algorithm should leverage knowledge of the hypothesis class as well, which i.i.d sampling is agnostic to. 


\textbf{Baseline: Active Learning:} An algorithm that achieves both direct and MP estimation accuracy is \emph{PAC active learning}~\cite{hanneke2014theory} (where PAC stands for Probably Approximately Correct~\cite{valiant1984theory}). PAC active learning algorithms guarantee that, with high probability, $\hat{h}$ in the resultant version space is such that $\PP(\hat{h}(x) \neq h^*(x)) \leq \pminor \epsilon = O(\epsilon)$. With this, we have $\abs{ \mu(\hat{h}) - \mu(h^*) } \leq \epsilon$ (see Lemma~\ref{lem:pac-auditing} in Appendix~\ref{sec:intro-deferred} for a formal proof).

To mention a setting where learning is favored over i.i.d sampling, learning homogeneous linear classifiers under certain well-behaved unlabeled data distributions requires only $O(d \log 1 / \epsilon)$ queries~\citep[e.g.][]{dasgupta2005coarse,balcan2013active} and would thus be far more efficient than $O( 1/ \epsilon^2)$ for low-dimensional learning settings with high auditing precision requirements. 

Still, as our goal is only to estimate the $\mu$ values of the induced version space, it is unclear if we need to go as far as to learn the model itself. In this paper, we investigate whether, and if so when, it may be possible to design adaptive approaches to efficiently directly and/or MP estimate $\mu(h^*)$ using knowledge of $\Hcal$.

To the best of our knowledge, we are the first to theoretically investigate  active approaches for direct and MP estimation of $\mu(h^*)$. Our first exploration of active fairness estimation seeks to provide a more complete picture of the theory of auditing machine learning models. Our hope is that our theoretical results can pave the way for subsequent development of practical algorithms.





\textbf{Our Contributions:}
Our main contributions are on two fronts, MP estimation and direct estimation of $\mu(h^*)$:

\begin{itemize}
    \item For the newly introduced notion of manipulation-proofness, we identify a statistically optimal, but computationally intractable deterministic algorithm. We gain insights into its query complexity through comparisons to the two baselines, i.i.d sampling and PAC active learning.
    
    \item In light of the computational intractability of the optimal deterministic algorithm, we design a randomized algorithm that enjoys \emph{oracle efficiency}~\citep[e.g.][]{dasgupta2007general}: it has an efficient implementation given access to a mistake-bounded online learning oracle, and an constrained empirical risk minimization oracle for the hypothesis class $\Hcal$. Furthermore, 
    its query performance matches that of the optimal deterministic algorithm up to $\polylog|\Hcal|$ factors.
    
    \item Finally, on the direct estimation front, we obtain bounds on information-theoretic query complexity. We establish that MP estimation may be more expensive than direct estimation, thus highlighting the need to develop separate algorithms for the two guarantees. Then, we establish the usefulness of randomization in algorithm design and develop an optimal, randomized algorithm for linear classification under Gaussian subpopulations. Finally, to shed insight on general settings, we develop distribution-free lower bounds for direction estimation under general VC classes. This lower bound charts the query complexity that any optimal randomized auditing algorithms must attain.
    
\end{itemize}


\subsection{Additional Notations} 
\label{sec:addl-notations}
We now introduce some additional useful notation used throughout the paper. Let $[m]$ denote $\{1, ..., m\}$. For an unlabeled dataset $S$, and two classifiers $h, h'$, we say $h(S) = h'(S)$ if for all $x \in S$, $h(x) = h'(x)$. Given a set of classifiers $V$ and a labeled dataset $T$, define $V[T] := \cbr{ h \in V: \forall (x,y) \in T, h(x) = y }$. Furthermore, denote by $V_x^{y} = V\sbr{ \cbr{ (x,y) } }$ for notational simplicity. 
Given a set of classifiers $V$ and fairness measure $\mu$, denote by $\diam_\mu(V) := \max_{h,h' \in V} \mu(h) - \mu(h')$ the \emph{$\mu$-diameter} of $V$. Given a set of labeled examples $T$, denote by $\Pr_T(\cdot)$ the probability over the uniform distribution on $T$; given a classifier $h$, 
denote by $\err(h, T) = \Pr_T(h(x) \neq y)$ the empirical error of $h$ on $T$. 

Throughout this paper, we will consider active fairness auditing under the membership
query model, similar to membership query-based active learning~\cite{angluin1988queries}. Specifically,
a deterministic active auditing  algorithm $\Acal$ with label budget $N$ is formally defined as a collection of $N+1$ (computable) functions $f_1, f_2, \ldots, f_N, g$ such that:

\begin{enumerate}
    \item For every $i \in [N]$, $f_i: (\Xcal \times \Ycal)^{i-1} \to \Xcal$ is the label querying function used at step $i$, that takes into input the first $(i-1)$ labeled examples $\langle (x_1,y_1), \ldots, (x_{i-1}, y_{i-1}) \rangle$ obtained so far, and chooses the $i$-th example $x_i$ for label query. 
    \item $g: (\Xcal \times \Ycal)^N \to \RR$ is the estimator function that takes into input all $N$ labeled examples $\langle (x_1, y_1), \ldots, (x_N, y_N) \rangle$ obtained throughout the interaction process, and outputs $\hat{\mu}$, the estimate of $\mu(h^*)$.
\end{enumerate}

When $\Acal$ interacts with a target classifier $h$, let the resultant queried unlabeled dataset be $S_{\Acal, h} = \langle x_1, \ldots, x_N \rangle$,
and the final $\mu$ estimate be $\hat{\mu}_{\Acal, h}$.

Similar to deterministic algorithms, a randomized active auditing algorithm $\Acal$ with label budget $N$ and $B$ bits of random seed is formally defined as a collection of $N+1$ (computable) functions $f_1, \ldots, f_N, g$, where $f_i: (\Xcal \times \Ycal)^{i-1} \times \cbr{0,1}^B \to \Xcal$ and $g: (\Xcal \times \Ycal)^N \times \cbr{0,1}^B \to \RR$. Note that each function now take as input a $B$-bit random seed; 
as a result, when $\Acal$ interacts with a fixed $h^*$, its output $\hat{\mu}$ is now a random variable. Note also that under the above definition, a randomized active auditing algorithm $\Acal$ that uses a fixed seed $b$ may be viewed as a deterministic active auditing algorithm $\Acal_b$. 


We will  be comparing our algorithms' query complexities with those of disagreement-based active learning algorithms~\cite{cohn1994improving,hanneke2014theory}.
Given a classifier $h$ and $r > 0$, define $B(h,r) = \cbr{h' \in \Hcal: \Pr_{D_X}\del{h'(x) \neq h(x)} \leq r}$ as the disagreement ball centered at $h$ with radius $r$. 
Given a set of classifiers $V$, define its disagreement region
$\DIS(V) = \cbr{x \in \Xcal: \exists h,h' \in V: h(x) \neq h'(x)}$. For a hypothesis class $\Hcal$ and an unlabeled data distribution $D_X$, an important quantity that characterizes the query complexity of disagreement-based active learning algorithm is the \emph{disagreement coefficient} $\theta(r)$, defined as
\[
\theta(r) = \sup_{h \in \Hcal, r' \geq r} \frac{\Pr_{D_X}( x \in \DIS( B(h, r') ) )}{r'}.
\]

\section{Related Work}



Our work is most related to the following two lines of work, both of which are concerned with estimating some property of a model without having to learn the model itself.

\textbf{Sample-Efficient Optimal Loss Estimation:} \citet{dicker2014variance,kong2018estimating} propose U-statistics-based estimators that estimate the optimal population mean square error in $d$-dimensional linear regression, with a sample complexity of $O(\sqrt{d})$ (much 
lower than $O(d)$, the sample complexity of learning optimal linear regressor). \citet{kong2018estimating} also extend the results to a well-specified logisitic regression setting, where the goal is to estimate the optimal zero-one loss.
Our work is similar in focusing on the question of efficient $\mu(h^*)$ estimation without having to learn $h^*$. Our work differs in focusing on fairness property instead of the optimal MSE or zero-one loss. Moreover, our results apply to arbitrary $\Hcal$, and not just to linear models.



\textbf{Interactive Verification:} \citet{goldwasser2021interactive} studies verification of whether a model $h$'s loss is near-optimal with respect to a hypothesis class $\Hcal$ and looks to understand when verification is cheaper than learning. They prove that verification is cheaper than learning for specific hypothesis classes and is just as expensive for other hypothesis classes. Again, our work differs in focusing on a different property of the model, fairness. 


Our algorithm also utilizes tools from active learning and machine teaching, which we review below.

\textbf{Active Learning and Teaching:} The task of learning $h^*$ approximately through membership queries has been well-studied~\citep[e.g.][]{angluin1988queries,hegedHus1995generalized,dasgupta2005analysis,hanneke2006cost,hanneke2007teaching}. Our computationally efficient algorithm for active fairness auditing is built upon the connection between active learning and machine teaching~\cite{goldman1995complexity}, as first noted in~\citet{hegedHus1995generalized,hanneke2007teaching}. To achieve computational efficiency, our work builds on recent work on black-box teaching \cite{dasgupta2019teaching}, which implicitly gives an efficient procedure for computing an approximate-minimum specifying set; we adapt~\citet{dasgupta2019teaching}'s algorithm to give a similar procedure for approximating the minimum specifying set that specifies the $\mu$ value.

In the interest of space, please see discussion of additional related work in Appendix~\ref{sec:addl-relwork}.

\section{Manipulation-Proof Algorithms}
\label{sec:mp}

\subsection{Optimal Deterministic Algorithm}


We begin our study of the MP estimation of $\mu(h^*)$ by identifying an optimal deterministic algorithm based on dynamic programming. Inspired by a minimax analysis of exact active learning with membership queries~\cite{hanneke2006cost}, we recursively define the following value function for any version space $V \subseteq \Hcal$:
\begin{align*}
\CC(V) = 
\begin{cases}
0, \qquad \qquad \qquad \qquad \qquad  \diam_\mu(V) \leq 2\epsilon \\
1 + \min_x \max_y \CC(V[(x,y)]), \text{otherwise} 
\end{cases}
\end{align*}
Note that $\CC(V)$ is similar to the minimax query complexity of exact active learning~\cite{hanneke2006cost}, except that the induction base case is different -- here the base case is $\diam_\mu(V) \leq 2\epsilon$, which implies that subject to $h^* \in V$, we have  identified $\mu(h^*)$ up to error $\epsilon$. In contrast, in exact active learning,~\citet{hanneke2006cost}'s induction base case is $|V| = 1$, where we identify $h^*$ through $V$.

The value function $\CC$ also has a game-theoretic interpretation: imagine that a learner plays a multi-round game with an adversary. The learner makes sequential queries of examples to obtain their labels, and the adversary reveals the labels of the examples, subject to the constraint that all labeled examples shown agree with some classifier in $\Hcal$. The version space $V$ encodes the state of the game: it is the set of classifiers that agrees with all the labeled examples shown so far in the game. 
The interaction between the learner and the adversary ends when all classifiers in $V$ has $\mu$ values $2\epsilon$-close to each other. The learner would like to minimize its total cost, which is the number of rounds. $\CC(V)$ can be viewed as the minimax-optimal future cost, subject to the game's current state being represented by version space $V$.

\begin{algorithm}[t]
\begin{algorithmic}[1]
\REQUIRE{Finite hypothesis class $\Hcal$, target error $\epsilon$, fairness measure $\mu$}
\ENSURE{$\hat{\mu}$, an estimate of $\mu(h^*)$}
\STATE Let $V \gets \Hcal$
\label{line:init-vs}
\WHILE{$\diam_\mu(V) > 2\epsilon$}
\STATE Query $x \in \argmin_x \max_y \CC\del{V_x^y}$, obtain label $h^*(x)$
\label{line:query-x}
\STATE $V \gets V(h^*, \cbr{x})$
\label{line:update-vs}
\ENDWHILE
\RETURN $\frac12 \del{ \max_{h \in V} \mu(h) + \min_{h \in V} \mu(h) }$
\label{line:return-est}
\end{algorithmic}
\caption{Minimax optimal deterministic auditing}
\label{alg:opt-det-auditing}
\end{algorithm}

Based on the notion of $\CC$, we design an algorithm, Algorithm~\ref{alg:opt-det-auditing}, that
has a worst-case label complexity at most $\CC(\Hcal)$. Specifically, it maintains a version space $V \subset \Hcal$, initialized to $\Hcal$ (line~\ref{line:init-vs}). At every iteration, if the $\mu$-diameter of $V$,  $\diam_\mu(V) = \max_{h,h' \in V} \mu(h) - \mu(h')$, is at most $2\epsilon$, then since $\mu(h^*) \in I = [ \min_{h \in V} \mu(h), \max_{h \in V} \mu(h) ]$ returning the midpoint of $I$ gives us an $\epsilon$-accurate estimate of $\mu(h^*)$ (line~\ref{line:return-est}). Otherwise, Algorithm~\ref{alg:opt-det-auditing} makes a query by choosing the $x$ that minimizes the worst-case future value functions (line~\ref{line:query-x}). After receiving $h^*(x)$, it updates its version space $V$ (line~\ref{line:update-vs}). 
By construction, the interaction between the learner and the labeler lasts for at most $\CC(V)$ rounds, which gives the following theorem.

\begin{theorem}
If Algorithm~\ref{alg:opt-det-auditing} interacts with some $h^* \in \Hcal$, then it outputs $\hat{\mu}$ such that 
$\abs{ \hat{\mu} - \mu(h^*) } \leq \epsilon$, and queries at most $\CC(\Hcal)$ labels. 
\label{thm:cost-det-al-ub}
\end{theorem}

By the minimax nature of $\CC$, 
we also show that among all deterministic algorithms,  Algorithm~\ref{alg:opt-det-auditing} has the optimal worst-case query complexity:

\begin{theorem}
If $\Acal$ is a deterministic algorithm  with query budget $N \leq \CC(\Hcal) - 1$, there exists some $h^* \in \Hcal$, such that $\hat{\mu}$, the output of $\Acal$ after querying $h^*$, satisfies $\abs{ \hat{\mu} - \mu(h^*) } > \epsilon$. 
\label{thm:cost-det-al-lb}
\end{theorem}

The proofs of Theorems~\ref{thm:cost-det-al-ub} and~\ref{thm:cost-det-al-lb} are deferred to Appendix~\ref{sec:pf-cost-det-al}.



\subsubsection{Comparison to Baselines}

To gain a better understanding of $\CC(\Hcal)$, we relate it to the label complexity of Algorithm~\ref{alg:opt-det-auditing} with those of the two baselines, i.i.d sampling and active learning. To establish the comparison, we prove that we can derandomize existing i.i.d sampling-based and active learning-based auditing algorithms with a small overhead on label complexity. The comparison follows as Algorithm~\ref{alg:opt-det-auditing} is the optimal deterministic algorithm.


Our first result is that, the label complexity of Algorithm~\ref{alg:opt-det-auditing} is within a factor of $O(\ln|\Hcal|)$ of the label complexity of i.i.d sampling. 

\begin{proposition}
\label{prop:cost-iid}
$\CC(\Hcal) \leq O\del{ \frac{1}{\epsilon^2} \ln |\Hcal| }$.
\end{proposition}

Our second result is that the label complexity of Algorithm~\ref{alg:opt-det-auditing}
is always no worse than the distribution-dependent label complexity of CAL~\cite{cohn1994improving,hanneke2014theory}, a well-known PAC active learning algorithm. We believe that similar bounds of $\CC(\Hcal)$ compared to generic active learning algorithms can be shown,
such as the Splitting Algorithm~\cite{dasgupta2005coarse} or the confidence-based algorithm of~\citet{zhang2014beyond}, through suitable derandomization procedures.

\begin{proposition}
\label{prop:cost-cal}
$\CC(\Hcal) \leq O\del{ \theta(\epsilon) \cdot \ln|\Hcal| \cdot \ln\frac1\epsilon }$, where $\theta$ is the disagreement coefficient of $\Hcal$ with respect to $D_X$ (recall Section~\ref{sec:addl-notations} for its definition).
\end{proposition}

\begin{proof}[Proof sketch]
We present Algorithm~\ref{alg:cal-derandomized}, which is a derandomized version of the Phased CAL algorithm~\citep[][Chapter 2]{hsu2010algorithms}. To prove this proposition, using Theorem~\ref{thm:cost-det-al-lb}, it suffices to show that Algorithm~\ref{alg:cal-derandomized} has a deterministic label complexity bound of $O\del{ \theta(\epsilon) \cdot \ln|\Hcal| \cdot \ln\frac1\epsilon }$. We only present the main idea here, and defer a precise version of the proof to Appendix~\ref{sec:cost-cal-deferred}.

We first show that for every $n$, the optimization problem in line~\ref{line:sample-select} is always feasible. To see this, observe that if we draw $S_n$, a sample of size $m_n$, drawn i.i.d from $D_X$, we have:
\begin{enumerate}
\item By Bernstein's inequality, with probability $1-\frac14$,
\[
\Pr_{S_n}( x \in \DIS(V_n)) \leq 2 \Pr_{D_X}( x \in \DIS(V_n) ) + \frac{\ln 8}{m_n},
\]
\item By Bernstein's inequality and union bound over $h, h' \in \Hcal$, we have with probability $1-\frac14$,
\begin{align*}
    \forall h, h' \in \Hcal: \;\;\;\; & \Pr_S( h(x) \neq h'(x) ) = 0 \\  
    \implies &   \Pr_{D_X}( h(x) \neq h'(x) ) \leq \frac{16 \ln|\Hcal|}{m_n}. 
\end{align*}
\end{enumerate}
By union bound, with nonzero probability, the above two condition hold simultaneously, showing the feasibility of the optimization problem.


We then argue that for all $n$, $V_{n+1} \subseteq B(h^*, \frac{16 \ln|\Hcal|}{m_n})$. This is
because for each $h \in V_{n+1}$, $h$ and $h^*$ are both in $V_n$ and therefore they agree on $S_n \setminus T_n$; on the other hand, $h$ and $h^*$ agree on $T_n$ by the definition of of $V_{n+1}$. As a consequence, 
$\Pr_{S_n}( h(x) \neq h^*(x) ) = 0$, which implies that $\Pr_{D_X}( h(x) \neq h^*(x) ) \leq \frac{16 \ln|\Hcal|}{m_n}$.
As a consequence, for all $h \in V_{N+1}$, $\Pr(h(x) \neq h^*(x)) \leq \pminor \epsilon$, which, combined with Lemma~\ref{lem:pac-auditing}, implies that $\abs{ \mu(h) - \mu(h^*) } \leq \epsilon$.

Finally, to upper bound Algorithm~\ref{alg:cal-derandomized}'s label complexity:
\begin{align*}
\sum_{n=1}^N |T_n|
= &
\sum_{n=1}^N m_n \cdot ( 2 \Pr_{D_X}( x \in \DIS(V_n) ) + \frac{\ln 8}{m_n} ) \\
\leq & \sum_{n=1}^N m_n \cdot ( 2 \theta(\epsilon)  \frac{16\ln|\Hcal|}{m_n} + \frac{\ln 8}{m_n}   ) \\
\leq & 
O\del{ \theta(\epsilon) \cdot \ln|\Hcal| \cdot \ln\frac1\epsilon }.
\qedhere
\end{align*}
\end{proof}

\begin{algorithm}[t]
\begin{algorithmic}[1]
\REQUIRE{Hypothesis class $\Hcal$, target error $\epsilon$, minority population proportion $\pminor$, fairness measure $\mu$}
\ENSURE{$\hat{\mu}$, an estimate of $\mu(h^*)$}
\STATE Let $N = \lceil \log_2 \frac{16 \ln|\Hcal|}{\pminor \epsilon} \rceil$.
\STATE Let $V_1 \gets \Hcal$
\FOR{$n=1,\ldots,N$}
\STATE Let $m_n = 2^n$
\STATE Find (the lexicographically smallest) $S_n \in \Xcal^{m_n}$ such that:
\[
\Pr_{S_n}( x \in \DIS(V_n)) \leq 2 \Pr_{D_X}( x \in \DIS(V_n) ) + \frac{\ln 8}{m_n},
\]
and
\begin{align*}
    \forall h, h' \in \Hcal: \;\;\;\; & \Pr_{S_n}( h(x) \neq h'(x) ) = 0  \\  
    \implies & \Pr_{D_X}( h(x) \neq h'(x) ) \leq \frac{16 \ln|\Hcal|}{m_n}. 
\end{align*}
\label{line:sample-select}
\STATE Query $h^*$ for the labels of examples in $T_n := S_n \cap \DIS(V_n)$
\STATE $V_{n+1} \gets V_n( h^*, T_n )$. 
\ENDFOR
\RETURN $\mu(h)$ for an arbitrary $h \in V_{N+1}$. 
\end{algorithmic}
\caption{Derandomized Phased CAL for Auditing}
\label{alg:cal-derandomized}
\end{algorithm}



\subsubsection{Computational Hardness of Implementing Algorithm~\ref{alg:opt-det-auditing}}

Although Algorithm \ref{alg:opt-det-auditing} has the optimal label complexity guarantees among all deterministic algorithms, we show in the following proposition that, under standard complexity-theoretic assumptions ($\mathrm{NP} \not \subseteq \mathrm{TIME}(n^{O(\log\log n)})$),
even approximating $\CC(\Hcal)$ is computationally intractable.
\begin{proposition}
If there is an algorithm that can approximate $\CC(\Hcal)$ to within $0.3 \ln|\Hcal|$ factor in $\poly(|\Hcal|, |\Xcal|, 1/\epsilon)$ time, then $\mathrm{NP} \subseteq \mathrm{TIME}(n^{O(\log\log n)})$.
\label{prop: costh-to-sc}
\end{proposition}

We remark that the constant 0.3 can be improved to a constant arbitrarily smaller than 1.
The main insight behind this proposition is a connection between $\CC(\Hcal)$ and optimal-depth decision trees (see Theorem~\ref{thm:cc-tree}): using the hardness of computing an approximately-optimal-depth decision tree~\cite{laber2004hardness} and taking into account the structure of $\mu$, we establish the intractability of approximating $\CC(\Hcal)$. 

Owing to the intractability of Algorithm \ref{alg:opt-det-auditing}, in the next section, we turn to the design of a computationally efficient algorithm whose label complexity nears that of Algorithm~\ref{alg:opt-det-auditing} (i.e. $\CC(\Hcal)$).

\subsection{Efficient Randomized Algorithm with Competitive Guarantees}
\label{sec:efficient-audit}

We present our efficient algorithm in this section, which also serves as a first upper bound on the statistical complexity of computationally tractable algorithms. Our algorithm, Algorithm~\ref{alg:efficient-audit}, is inspired  by the exact active learning literature~\citep{hegedHus1995generalized, hanneke2007teaching}, based on the connection between machine teaching~\cite{goldman1995complexity} and active learning. 

Algorithm~\ref{alg:efficient-audit} takes into input two oracles, a mistake-bounded online learning oracle $\Ocal$ and an constrained empirical risk minimization (ERM) oracle $\CERM$, defined below.

\begin{definition}
An {\em online-learning oracle} $\Ocal$ is said to have a mistake bound of $M$ for hypothesis class $\Hcal$, if for any classifier $h^* \in \Hcal$, and any sequence of examples $x_1, x_2, \ldots $, at every round $t \in \NN$, given historical examples $(x_s, h^*(x_s))_{s=1}^{t-1}$, 
outputs classifier $\hat{h}_t$ such that $\sum_{t=1}^\infty I(\hat{h}_t(x_t) \neq h^*(x_t)) \leq M$.
\end{definition}
Well-known implementations of mistake bounded online learning oracle include the halving algorithm and its efficient sampling-based approximations~\cite{bertsimas2004solving} as well as the Perceptron / Winnow algorithm~\cite{littlestone1988learning,ben2009agnostic}. 
For instance, if $\Ocal$ is the halving algorithm, a mistake bound of $M = \log_2|\Hcal|$ may be achieved.

We next define the constrained ERM oracle, which has been previously used in a number of works on oracle-efficient active learning~\cite{dasgupta2007general,hanneke2011rates,huang2015efficient}.

\begin{definition}
An \emph{constrained ERM oracle} for hypothesis class $\Hcal$, $\CERM$, is one that takes as input labeled datasets $A$ and $B$, and outputs a classifier $\hat{h} \in \argmin\cbr{ \err(h, A): h \in \Hcal, \err(h, B) = 0 }$.  
\end{definition}

The high-level idea of Algorithm~\ref{alg:efficient-audit} is as follows: at every iteration, it uses the mistake-bounded online learning oracle to generate some classifier $\hat{h}$ (line~\ref{line:ol-oracle}); then, it aims to construct a dataset $T$ of small size, such that after querying $h^*$ for the labels of examples in $T$, one of the following two happens: (1) $\hat{h}$ disagrees with $h^*$ on some example in $T$; (2) for all classifiers in the version space $V = \cbr{h \in \Hcal: \forall x \in T, h(x) = h^*(x)}$, we have $\diam_\mu(V) \leq 2\epsilon$. In case (1), we have found a counterexample for $\hat{h}$, which can be fed to the online learning oracle to learn a new model, and this can happen at most $M$ times; in case (2), we are done: our queried labeled examples ensure that our auditing estimate is $\epsilon$-accurate, and satisfies manipulation-proofness. Dataset $T$ of such property is called a \emph{$(\mu,\epsilon)$-specifying set} for $\hat{h}$, as formally defined in Defintion~\ref{def:mu-spec-set} in Appendix~\ref{sec:hegedus}. 


Another view of the $\mu$-specifying set is a set $T$ such that for all $h,h'$ with $\mu(h) - \mu(h') > 2\epsilon$, there exists some $x \in T$, such that $h(x) \neq \hat{h}(x)$ or $h'(x) \neq \hat{h}(x)$. The requirements on $T$ can be viewed as a set cover problem, where the universe $U$ is $\cbr{ (h,h') \in \Hcal^2: \mu(h) - \mu(h') > 2\epsilon }$, and the set system is $\Ccal = \cbr{ C_x: x \in \Xcal }$, where $(h,h')$ is in $C_x$ if $h(x) \neq \hat{h}(x)$ or $h'(x) \neq \hat{h}(x)$. 

This motivates us to design efficient set cover algorithms in this context. A key challenge of applying standard offline set cover algorithms (such as the greedy set cover algorithm) to construct approximate minimum $(\mu,\epsilon)$-specifying set is that we cannot afford to enumerate all elements in the universe $U$: $U$ can be exponential in size. 

In face of this challenge, we draw inspiration from online set cover literature~\cite{alon2009online,dasgupta2019teaching} to design an oracle-efficient algorithm that computes $O( \log|\Hcal|\log|\Xcal|)$-\emph{approximate} minimum  $(\mu,\epsilon)$-specifying sets, which avoids enumeration over $U$.

Our key idea is to simulate an online set cover process. We build the cover set\footnote{When it is clear from context, we slightly abuse notations and say ``$x$ covers $(h,h')$'' if $(h,h') \in C_x$.} $T$ iteratively, starting from $T = \emptyset$ (line~\ref{line:init-t}). At every inner iteration, we first try to find a pair $(h_1,h_2)$ in $U$ not yet covered by the current $T$. As we shall see next, this step (line~\ref{line:cerm}) can be implemented efficiently given the constrained ERM oracle $\CERM$. If such a pair $(h_1,h_2)$ can be found, we use the online set cover algorithm implicit in~\cite{dasgupta2019teaching} to find a new example that covers this pair, add it to $T$,  and move onto the next iteration (lines~\ref{line:start-add-t} to~\ref{line:end-add-t}). Otherwise, $T$ has successfully covered all the elements in $U$, in which case we break the inner loop (line~\ref{line:return-1}).   

To see how line~\ref{line:cerm} finds an uncovered pair in $U$, we note that it can be also written as:
\[
(h_1 ,h_2) = \argmax_{h,h' \in \Hcal} \cbr{\mu(h) - \mu(h'): h(T) = h'(T) = \hat{h}(T)}
\]
Thus, if $\mu(h_1) - \mu(h_2) > 2\epsilon$, then the returned pair $(h_1, h_2)$ corresponds to a pair in universe $U$ that is not covered by $T$. Otherwise, by the optimality of $(h_1,h_2)$, $T$ covers all elements in $U$. 

Furthermore, we note that optimization problems~\eqref{eqn:h-1} and~\eqref{eqn:h-2} can be implemented with access to $\CERM$. We show this for program~\eqref{eqn:h-1} and the reasoning for program~\eqref{eqn:h-2} is analogous. Observe that maximizing $\mu(h)$ from $h \in \Hcal$ subject to constraint $h(T) = \hat{h}(T)$ is equivalent to minimizing (a weighted) empirical error of $h \in \Hcal$ on dataset $\cbr{ (x,+1): x \in \Xcal, x_A =0 } \cup \cbr{(x,-1): x \in \Xcal, x_A = 1}$, subject to $h$ having zero error on $\{(x, \hat{h}(x)): x \in T\}$.









We are now ready to present the label complexity guarantee of Algorithm~\ref{alg:efficient-audit}.

\begin{algorithm}[thb!]
\caption{Oracle-efficient Active Fairness Auditing}
\label{alg:efficient-audit}
\begin{algorithmic}[1]
\REQUIRE Hypothesis class $\Hcal$, online learning oracle $\Ocal$ with mistake bound $M$, constrained ERM oracle $\CERM$, target error $\epsilon$, fairness measure $\mu$.
\ENSURE{$\hat{\mu}$, an estimate of $\mu(h^*)$}

\STATE Initialize $S \gets \emptyset$

\WHILE{\TRUE}
\STATE $\hat{h} \gets \Ocal(S)$ 
\label{line:ol-oracle}
\STATE Let $T \gets \emptyset$\\
\label{line:init-t}
\COMMENT{Computing an approximate minimum $(\mu,\epsilon)$-specifying set for $\hat{h}$}

\STATE Initialize weights $w(x) = \frac{1}{|\Xcal|}$ and threshold $\thres_x \sim \Exp(\ln(|\Hcal|^2 M/\delta))$ \COMMENT{random initialization of thresholds}
\WHILE{\textbf{true}}
\STATE Use $\CERM$ to solve separate programs:
\begin{equation} 
h_1 \gets \text{find} \; \max_{h \in \Hcal} \mu(h), \text{s.t.} \; h(T) = \hat{h}(T)
\label{eqn:h-1}
\end{equation}
and 
\begin{equation}
h_2 \gets \text{find} \; \min_{h \in \Hcal} \mu(h), \text{s.t.} \; h(T) = \hat{h}(T)
\label{eqn:h-2}
\end{equation} 
\label{line:cerm}
\COMMENT{$T$ is an $(\mu,\epsilon)$-specifying set for $\hat{h}$}
\IF{$\mu(h_1) - \mu(h_2) \leq 2\epsilon$}

\STATE \textbf{break}
\label{line:return-1}
\ELSE
\STATE \COMMENT{Add examples to $T$ to cover $(h_1, h_2)$, using the online set cover algorithm implicit in~\cite{dasgupta2019teaching}} \\
Determine $\Delta(h_1, h_2) = \{x \in \Xcal: h_1(x) \neq \hat{h}(x) \text{ or } h_2(x) \neq \hat{h}(x) \}$ 
\label{line:start-add-t}

\WHILE{$\sum_{x \in \Delta(h_1, h_2)} w(x) \leq 1$}
\STATE Double weights $w(x)$ for all $x$ in $\Delta(h_1, h_2)$ 
\label{line:doubling}
\STATE Update $T \gets \cbr{ x \in \Xcal: w(x) \geq \thres_x }$
\label{line:end-add-t}
\ENDWHILE
\ENDIF
\ENDWHILE
\STATE Query $h^*$ on $T$
\STATE $S \gets S \cup T$
\IF{$\hat{h}(T) = h^*(T)$}
\label{line:return-condition}
\RETURN $\frac12 (\mu(h_1) + \mu(h_2))$
\ENDIF
\ENDWHILE
\end{algorithmic}
\end{algorithm}

\begin{theorem}
\label{thm:auditing-complexity-oe}
If the online learning oracle $\Ocal$ makes a total of $M$ mistakes, then with probability $1-\delta$,
Algorithm~\ref{alg:efficient-audit} outputs $\hat{\mu}$ such that 
$\abs{ \hat{\mu} - \mu(h^*) } \leq \epsilon$, with its number of label queries bounded by:
\[
O\del{ \CC(\Hcal) M \log \frac{|\Hcal| M}{\delta} \log |\Xcal|}.
\]
\end{theorem}

The proof of Theorem~\ref{thm:auditing-complexity-oe} is deferred to Appendix~\ref{sec:hegedus}. In a nutshell, it combines the following observations. First, Algorithm~\ref{alg:efficient-audit} has at most $M$ outer iterations using the mistake bound guarantee of oracle $\Ocal$. Second, for each $\hat{h}$ in each inner iteration, its minimum $(\mu,\epsilon)$-specifying set has size at most $\CC(\Hcal)$; this is based on a nontrivial connection between the optimal deterministic query complexity and $(\mu, \epsilon)$-extended teaching dimension (see Definition~\ref{def:mu-xtd}), which we present in Lemma~\ref{lem:cost-td}. Third, by the $O\del{ \log \frac{|\Hcal| M}{\delta} \log |\Xcal| }$-approximation guarantee of the online set cover algorithm implicit in~\cite{dasgupta2019teaching}, each outer iteration makes at most $O\del{ \CC(\Hcal) \log \frac{|\Hcal| M}{\delta} \log |\Xcal| }$ label queries. 




\begin{remark} As we have seen, Algorithm~\ref{alg:efficient-audit} implicitly performs online set cover. With this connection, it inherits the $\tilde{\Omega}(\log|\Hcal| \log|\Xcal|)$ inapproximability factor of online set cover~\citep[][Proposition 4.2]{alon2009online}.

\end{remark}



Finally, in Appendix~\ref{sec:experiments}, we empirically explore the performance of Algorithm~\ref{alg:efficient-audit} and active learning, and compare them with i.i.d sampling. As expected, our experiments confirm that under a fixed budget, Algorithm~\ref{alg:efficient-audit} is most effective at inducing a version space with a small $\mu$-diameter, and can thus provide the strongest manipulation-proofness guarantee. 

\section{Statistical Limits of Estimation}
\label{sec:estimation}

In this section, we turn to direct estimation, the second of the two main guarantees we wish to have for our auditing algorithm. In particular, we focus on the statistical limits of direct estimation, where the goal is to design an auditing algorithm that can output $\hat{\mu}$ such that $\abs{\hat{\mu} - \mu(h^*)} \leq \epsilon$ with a small number of queries.


\subsection{Separation between Estimation with and without Manipulation-proofness}
\label{sec:sep-mp}

To start, it is natural to contrast the guarantee of $\epsilon$-manipulation-proofness against $\epsilon$-estimation accuracy. Indeed, if the two guarantees are one and the same, we may just apply our auditing algorithms developed to achieve MP for direct estimation as well.

Here we look to answer the question of whether achieving MP is strictly harder, and we answer this question in the affirmative. Specifically, the following simple example suggests that MP estimation can sometimes require a much higher label complexity than direct estimation.

\begin{example}
Let $\epsilon = \frac14$ and $n \gg 1$. $\Xcal = \cbr{0,1,\ldots,n}$, and $x \mid x_A = 0 \sim \Uniform(\cbr{0})$, and $x \mid x_A = 1 \sim \Uniform(\cbr{1,\ldots,n})$. Let $\Hcal = \cbr{ h: \Xcal \to \cbr{-1,+1}, h(0) = -1 }$. 

First, as $\epsilon = \frac14$, the iid sampling baseline makes $O(1)$ queries and ensures that it estimates $\mu(h^*)$ with error at most $\epsilon$ with probability $\geq 0.9$. 

However, for manipulation-proof estimation, at least $\Omega(n)$ labels are needed to ensure that the queried dataset $S$ satisfies $\diam_\mu(\Hcal(h^*, S)) \leq \epsilon$. 
Indeed, let $h^* \equiv -1$. For any unlabeled dataset $S$ of size $\leq n/2$, 
by the definition of $\Hcal$,
there always exist $h, h' \in \Hcal(h^*, S)$, such that for all $x \in \cbr{1,\ldots,n} \setminus S$, $h(x) = -1$ and $h'(x) = +1$. 
As a result, 
$
\mu(h) = \frac{0}{n} - \frac{0}{1} = 0
$,
and
$
\mu(h') = \frac{\abs{\cbr{1,\ldots,n} \setminus S }}{n} - \frac{0}{1} 
\geq \frac{1}{2}
$, which implies that 
$\diam_\mu(\Hcal(h^*, S)) \geq \frac12 > \epsilon$. \qed
\label{ex:shattered}
\end{example}

\subsection{Randomized Algorithms for Direct Estimation}
\label{sec:randomization}

The separation result above suggests that different algorithms may be needed if we are only interested in efficient direct estimation. Motivated by our previous exploration, a first question to answer is whether randomization should be a key ingredient in algorithm design. That is, can a randomized auditing algorithm have a lower query complexity than that of the optimal deterministic algorithm? Using the example below, we answer this question in the affirmative. 

\begin{example}
Same as the setting of Example~\ref{ex:shattered}; recall that iid sampling, a randomized algorithm, estimates $\mu(h^*)$ with error at most $\epsilon = \frac14$ with probability $\geq 0.9$; it has a query complexity of $O(1)$.

In contrast, consider any deterministic algorithm $\Acal$ with label budget $N \leq \frac n 2$; we consider its interaction history with classifier $h_0 \equiv -1$, which can be summarized by a sequence of unlabeled examples $S = \langle x_1, \ldots, x_N \rangle$. 
Now, consider an alternative classifier $h_1$ such that $h_1(x) = -1$ on $S \cup \cbr{0}$, but $h_1(x) = +1$ on $\cbr{1,\ldots,n} \setminus S$. By an inductive argument, it can be shown that the interaction history between $\Acal$ and $h_1$ is also $S$, which implies that when the underlying hypotheses $h^* = h_0$ and $h^* = h_1$, $\Acal$ must output the same estimate $\hat{\mu}$ (see Lemma~\ref{lem:int-history-agree} in Appendix~\ref{sec:int-history-agree} for a formal proof); however, $\mu(h_0) - \mu(h_1) \geq \frac12$, implying that under at least one of the two hypotheses, we must have $\abs{\hat{\mu} - \mu(h^*)} \geq \frac14 = \epsilon$. 

In summary, in this setting, a randomized algorithm has a query complexity of $O(1)$, much smaller than $\Omega(n)$, the optimal query complexity of deterministic algorithms.
\qed
\end{example}



\subsection{Case Study: Non-homogeneous Linear Classifiers under Gaussian Populations}
\label{sec:halfspace-main}

In this subsection, we identify a  practically-motivated setting, where we are able to comprehensively characterize the minimax (randomized) active fairness auditing query complexity up to logarithmic factors. Specifically, we present a positive result in the form of an algorithm that has a query complexity of $\tilde{O}\del{\min(d, \frac1{\epsilon^2})}$, as well as a matching lower bound that shows any (possibly randomized) algorithm must have a query complexity of $\Omega\del{\min(d, \frac1{\epsilon^2})}$. 



\begin{example}
Let $d \geq 2$ and $\Xcal = \RR^d$. $x \mid x_A = 0 \sim \N(m_0, \Sigma_0)$, whereas $x \mid x_A = 1 \sim \N(m_1, \Sigma_1)$. Let hypothesis class $\Hcal_{\hs} = \cbr{ h_{a,b}(x):= \sign(\inner{a}{x} + b): a \in \RR^d, b \in \RR }$ be the class of non-homogenenous linear classifiers.

Recall that i.i.d sampling has a label complexity of $O\del{ \frac{1}{\epsilon^2}}$; on the other hand, through a membership query-based active learning algorithm (Algorithm~\ref{alg:audit-nonhom-full} in Appendix~\ref{sec:halfspace}), we can approximately estimate $\mu(h^*)$ (up to scaling) by doing $d$-binary searches, using active label queries. This approach incurs a total label complexity of $\tilde{O}(d)$. Choosing the better of these two algorithms gives an active fairness auditing strategy of label complexity $\tilde{O}\del{\min(d, \frac{1}{\epsilon^2})}$. 

We only present the main idea of Algorithm~\ref{alg:audit-nonhom-full}
here, with its full analysis deferred to Appendix~\ref{sec:halfspace}.
Its core component is Algorithm~\ref{alg:audit-nonhom} below, which label-efficiently estimates $\gamma(h^*) = \PP_{x \sim \N(0,I_d)}( h^*(x) = +1 )$, with black-box label queries to $h^*(x) = \sign(\inner{a^*}{x} + b^*)$. 
Algorithm~\ref{alg:audit-nonhom} is based on the following insights. First, observe that $\gamma(h^*) = \Phi\del{ \frac{b^*}{\| a^* \|_2} } =:
\Phi(sr)$,
where 
$\Phi$ is the standard normal CDF, $s := \sign(b^*)$, and $r := \sqrt{\frac{1}{\sum_{i=1}^d m_i^{-2}}}$, for
$m_i := -\frac{b^*}{a_i^*}$.
On the one hand, $s$ can be easily obtained by querying $h^*$ on $\vec{0}$ (line~\ref{line:query-zero}).
On the other hand, estimating $r$ can be reduced to estimating
each $m_i$. However, some $m_i$'s can be unbounded, which makes their estimation challenging. To get around this challenge, we prove the following lemma, which shows that it suffices to accurately estimate those $m_i$'s that are not  unreasonably large (i.e. $m_i$'s for $i \in S$, defined below): 



\begin{lemma}
\label{lem:hat-r-r}
Let $\alpha := \sqrt{2d \ln\frac1\epsilon}$ and $\beta := 2 d^{\frac52} (\ln\frac{1}{\epsilon})^{\frac 34} (\frac1\epsilon)^{\frac12}$. 
Suppose $r \leq \alpha$.
If there is some $S \subset [d]$, such that:
\begin{enumerate}
    \item for all $i \notin S$, $\abs{m_i} \geq \beta$,
    \item for all $i \in S$, $|\hat{m}_i - m_i| \leq \epsilon$;
\end{enumerate}
then, 
$
\abs{ \sqrt{\frac{1}{\sum_{i \in S} \hat{m}_i^{-2}} } - r } \leq 2\epsilon.
$
\end{lemma}

Algorithm~\ref{alg:audit-nonhom} carefully utilizes this lemma to estimate $r$. First, it tests whether for all $i$, $h^*(\alpha e_i)=h^*(-\alpha e_i)$; if yes, for all $i$, $\abs{m_i} \geq \alpha$, and $r \geq \sqrt{\ln\frac1\epsilon}$, and $\gamma(h^*)$ is $\epsilon$-close to 0 or 1 depending on the value of $s$ (line~\ref{line:query-alpha-end}). Otherwise, it must be the case that $r \leq \alpha$. In this case, we go over each coordinate $i$, first testing whether $\abs{m_i} \leq \beta$ (line~\ref{line:test-m-i-beta}); if no, we skip this coordinate (do not add it to $S$); otherwise, we include $i$ in $S$ and estimate $m_i$ to precision $\epsilon$ using binary search (line~\ref{line:bs}). By the guarantees of Lemma~\ref{lem:hat-r-r}, we have $\abs{s\hat{r} - sr} \leq 2\epsilon$, which, by the $\frac{1}{\sqrt{2\pi}}$-Lipschitzness of $\Phi$, implies that $\abs{\hat{\gamma} - \gamma(h^*)} \leq \epsilon$. The total query complexity of Algorithm~\ref{alg:audit-nonhom} is $1 + 2d + 2d + d \log_2 \frac{\beta}{\epsilon} = \tilde{O}(d)$. 


\begin{algorithm}[htb!]
\caption{\estimg: A label efficient estimation algorithm for $\gamma(h^*)$ for non-homogenoeus linear classifiers}
\label{alg:audit-nonhom}
\begin{algorithmic}[1]
\REQUIRE{query access to $h^* \in \Hcal_{\hs}$,
target error $\epsilon$.}
\ENSURE{$\hat{\gamma}$ such that $\abs{\hat{\gamma} - \gamma(h^*)} \leq \epsilon$.}
\STATE Let $\alpha = \sqrt{2 d \ln\frac1\epsilon}$, $\beta = 2 d^{\frac52} (\ln\frac{1}{\epsilon})^{\frac 34} (\frac1\epsilon)^{\frac12}$. 

\STATE $s \gets$ Query $h^*$ on $\vec{0}$
\label{line:query-zero}

\STATE Query $h^*$ on $\cbr{ \rho \alpha e_i: \rho \in \cbr{\pm 1}, i \in [d]}$
\label{line:query-alpha-start}

\IF{for all $i \in [d]$, $h^*(\alpha e_i)=h^*(-\alpha e_i)$}

\RETURN $1$ if $s = +1$, $0$ if $s = -1$
\label{line:query-alpha-end}

\ENDIF

\COMMENT{Otherwise, $r \leq \alpha = \sqrt{2d \ln\frac1\epsilon}$}

\STATE $S \gets \emptyset$
\label{line:s-m-est-start}

\FOR{$i=1, \ldots, d$}

\STATE Query $h^*$ on $\beta e_i$ and $-\beta e_i$ \label{line:query-beta}
\label{line:test-m-i-beta}

\IF{$h^*(\beta e_i) \neq h^*(-\beta e_i)$}

\STATE $S \gets S \cup \cbr{i}$

\COMMENT{Use binary search to obtain $\hat{m}_i$, an estimate of $m_i = -\frac{b^*}{a_i^*}$ with precision $\epsilon$}

\STATE $\hat{m}_i \gets \bs(i, \beta, \epsilon)$ (Algorithm~\ref{alg:binary-search})
\label{line:bs}
\label{line:s-m-est-end}
\ENDIF

\ENDFOR


\STATE $\hat{r} \gets \sqrt{ \frac{1}{\sum_{i \in S} \hat{m}_i^{-2}}}$ \COMMENT{$\hat{r}$ is an estimate of $r$}
\label{line:estimate-abs-b}
\STATE \textbf{return} $\Phi( s \hat{r} )$
\label{line:return-final}
\end{algorithmic}
\end{algorithm}

\begin{algorithm}
\caption{\bs}
\label{alg:binary-search}
\begin{algorithmic}[1]
\REQUIRE{$i, \beta$ such that $h^*(\beta e_i) \neq h^*(-\beta e_i)$, precision $\epsilon$}
\ENSURE{$m$, an $\epsilon$-accurate estimate of $m_i = -\frac{b}{a_i}$}
\STATE $u \gets \beta, l \gets -\beta$
\WHILE{ $u - l \geq \epsilon$ }
\STATE $m \gets \frac{u+l}{2}$
\STATE Query $h^*$ on $m e_i$
\IF{$h^*(m e_i) = h^*( l e_i )$}
\STATE $l \gets m$
\ELSE 
\STATE $u \gets m$
\ENDIF
\ENDWHILE
\RETURN $m$
\end{algorithmic}
\end{algorithm}


For the lower bound, we formulate a hypothesis testing problem, such that under hypotheses $H_0$ and $H_1$, the $\mu(h^*)$ values are approximately $\epsilon$-separated. This is used to show that any active learning algorithm with label query budget $\leq \Omega\del{\min(d, \frac1{\epsilon^2})}$ cannot effectively distinguish $H_0$ and $H_1$. Our construction requires a delicate analysis on the KL divergence between the observation distributions under the two hypotheses, and we refer the readers to Theorem~\ref{thm:lb-halfspace} for details.
\qed
\end{example}


\subsection{General Distribution-Free Lower Bounds}

Finally, in this subsection, we move beyond the Gaussian population setting and derive general query complexity lower bounds for randomized estimation algorithms that audit general hypothesis classes with finite VC dimension $d$. This result suggests that, when $d \gg \frac1{\epsilon^2}$, or equivalently $\epsilon \gg \frac{1}{\sqrt{d}}$, there exists some hard data distribution and target classifier in $\Hcal$, such that active fairness auditing has a query complexity lower bound of $\Omega(\frac{1}{\epsilon^2})$; that is, iid sampling is near-optimal.

\begin{theorem}[Lower bound for randomized auditing]
\label{thm:audit-lb-vc-main}
Fix $\epsilon \in (0,\frac{1}{40}]$ and a hypothesis class $\Hcal$ with VC dimension $d \geq 1600$.
For any (possibly randomized) algorithm $\Acal$ with label budget $N \leq O( \min(d, \frac1{\epsilon^2}) )$, there exists a distribution $D_X$ over $\Xcal$ and $h^* \in \Hcal$, such that $\Acal$'s output $\hat{\mu}$ when interacting with $h^*$, satisfies:
\[
\PP\del{ \abs{ \hat{\mu} - \mu(h^*) } > \epsilon } > \frac18
\]
\end{theorem}

The proof of Theorem~\ref{thm:audit-lb-vc-main} can be found at  Appendix~\ref{sec:audit-vc}. The lower bound construction follows from a similar setting as in Example~\ref{ex:shattered}, except that we now choose $h^*$ in a randomized fashion.

\section{Conclusion}
In this paper, we initiate the study of the theory of query-efficient algorithms for auditing model properties of interest. We focus on auditing demographic parity, one of the canonical fairness notions. We investigate the natural auditing guarantee of estimation accuracy, and introduce a new guarantee based on the possibility of post-audit manipulation: manipulation-proofness. We identify an optimal deterministic algorithm, a matching randomized algorithm and develop upper and lower bounds that mark the performance that any optimal auditing algorithm must meet. Our first exploration of active fairness estimation seeks to provide a more complete picture of the theory of auditing. A natural next direction is to explore guarantees for other fairness notions (such as equalized odds). Indeed, how does one construct query-efficient algorithms when $\mu$ is a function of both $h^*(x)$ and $y$? 
Another natural question, motivated by the connection to disagreement-based active learning, is to design active fairness auditing algorithms based on some notion of disagreement with respect to $\mu$.


\noindent\textbf{Acknowledgments.} We thank Stefanos Poulis for sharing the implementation of the black-box teaching algorithm of~\citet{dasgupta2019teaching}, and special thanks to Steve Hanneke and Sanjoy Dasgupta for helpful discussions. We also thank the anonymous ICML reviewers for their feedback.

\nocite{langley00}

\bibliographystyle{icml2022}
\bibliography{refs}

\newpage
\appendix
\onecolumn


\section{Additional Related Works}
\label{sec:addl-relwork}


\textbf{Property Testing:} Our notion of auditing that leverages knowledge of $\Hcal$ is similar in theme to the topic of property testing \cite{goldreich1998property,ron2008property,balcan2012active,blum2018active,blanc2020estimating,blais2021vc} which tests whether $h^*$ is in $\Hcal$, or $h^*$ is far away from any classifier in $\Hcal$, given query access to $h^*$. These works provide algorithms with testing query complexity of lower order than sample complexity for learning with respect to $\Hcal$, for specific hypothesis classes such as monomials, DNFs, decision trees, linear classifiers, etc. Our problem can be reduced to property testing by testing whether $h^*$ is in $\cbr{h \in \Hcal: \mu(h) \in [i\epsilon, (i+1)\epsilon]}$ for all $i \in \cbr{0,1,\ldots, \lceil \frac 1 \epsilon \rceil}$; however, to the best of our knowledge, no such result is known in the context of property testing. 

\textbf{Feature Minimization Audits:} \citet{rastegarpanah2021auditing} study another notion of auditing, focusing on assessing whether the model is trained inline with the GDPR's Data Minimization principle. Specifically, this work evaluates the necessity of each individual feature used in the ML model, and this is done by imputing each feature with constant values and checking the extent of variation in the predictions. One commonality with our work, and indeed across all auditing works, is the concern with minimizing the number queries needed to conduct the audit. 

\textbf{Herding for Sample-efficient Mean Estimation:} Additionally, the estimation of DP may be viewed as estimating the difference of two means. Viewed in this light, herding~\cite{xutowards} offers a way to use non-iid sampling to more efficiently estimate means. However, the key difference needed in herding is that $h^*$, whose output is $\{-1,1\}$, may be well-approximated by $\langle w, \phi(x)\rangle$ for some mapping $\phi$ known apriori.


\textbf{Comparison with~\citet{sabato2013auditing}:} Lastly,~\citet{sabato2013auditing} also uses the term ``auditing'' in the context of active learning with outcome-dependent query costs; although the term ``auditing'' is shared, our problem settings are completely different: \cite{sabato2013auditing} focuses on active learning the model $h^*$ as opposed to just estimating $\mu(h^*)$.


\section{A General Lemma on Deterministic Query Learning}
\label{sec:int-history-agree}

In this section, we present a general lemma inspired by~\citet{hanneke2007teaching}, which are used in our proofs for  establishing lower bounds on deterministic active fairness auditing algorithms. 

\begin{lemma}
\label{lem:int-history-agree}
If an deterministic active auditing algorithm $\Acal$ with label budget $N$ interacts with labeling oracle that uses classifier $h_0$, and generates the following interaction history: $\langle (x_1, h_0(x_1)), (x_2, h_0(x_2)), \ldots, (x_N, h_0(x_N)) \rangle$, and there exists a classifier $h_1$ such that $h_1(x) = h_0(x)$ for all $x \in \cbr{x_1, \ldots, x_N}$.
Then $\Acal$, when interacting with $h_1$, generates the same interaction history, and outputs the same auditing estimate; formally, $S_{\Acal, h_1} = S_{\Acal, h_0}$ and $\hat{\mu}_{\Acal, h_1} = \hat{\mu}_{\Acal, h_0}$.
\end{lemma}
\begin{proof}
Recall from Section~\ref{sec:addl-notations} that deterministic active auditing algorithm $\Acal$ can be viewed as a sequence of $N+1$ functions $f_1, f_2, \ldots, f_N, g$, where $\cbr{f_i}_{i=1}^N$ are the label query function used at each iteration, and $g$ is the final estimator function. 
We show by induction that for steps $i = 0,1,\ldots,N$, 
the interaction histories of $\Acal$ with $h_0$ and $h_1$ agree on their first $i$ elements.



\paragraph{Base case.} For step $i=0$, both interaction histories are empty and  agree trivially. 


\paragraph{Inductive case.} Suppose that the statement holds for step $i$, i.e.
$\Acal$, when interacting with both $h_0$ and $h_1$, generates the same set of labeled examples 
\[
S_i = \langle (x_1, y_1), \ldots, (x_i, y_i) \rangle,
\]
up to step $i$. 

Now, at step $i+1$, $\Acal$ applies the query function $f_{i+1}$ and queries the same example 
$x_{i+1} = f_{i+1}(S_i)$. By assumption of this lemma, $h_1(x_{i+1}) = h_0(x_{i+1})$, which implies that the $(i+1)$-st labeled example obtained when $\Acal$ interacts with $h_1$, $(x_{i+1}, h_1(x_{i+1}))$ is identical to $(x_{i+1}, h_1(x_{i+1}))$, the $(i+1)$-st example when $\Acal$ interacts with $h_0$. Combined with the inductive hypotheses that the two histories agree on the first $i$ examples, we have shown that $\Acal$, when interacting with $h_0$ and $h_1$, generates the same set of labeled examples 
\[
S_{i+1} = \langle (x_1, y_1), \ldots, (x_i, y_i), (x_{i+1}, y_{i+1}) \rangle
\]
up to step $i+1$. 

This completes the induction. 

As the interaction histories $\Acal$ with $h_0$ and $h_1$ are identical, the unlabeled data part of the history are identical, formally,  $S_{\Acal, h_1} = S_{\Acal, h_0}$. In addition, as in both interactive processes, $\Acal$ applies deterministic function $g$ to the same interaction history of length $N$ to obtain estimate $\hat{\mu}$, we have
$\hat{\mu}_{\Acal, h_1} = \hat{\mu}_{\Acal, h_0}$.
\end{proof}

\section{Deferred Materials from Section~\ref{sec:intro}}
\label{sec:intro-deferred}

The following lemma formalizes the idea that PAC learning with $O(\epsilon)$ error is sufficient for fairness auditing, given that $\pminor = \min\del{\Pr_{D_X}(x_A = 0), \Pr_{D_X}(x_A = 1)}$ is $\Omega(1)$.

\begin{lemma}
\label{lem:pac-auditing}
If $h$ is such that $\PP(h(x) \neq h^*(x)) \leq \alpha$, then 
$\abs{ \mu(h) - \mu(h^*) } \leq \frac{\alpha}{\underline{p}}$. 
\end{lemma}

\begin{proof}
First observe that
\begin{align*}
& \abs{\Pr(h(x) = +1 \mid x_A = 0) - \Pr(h^*(x) = +1 \mid x_A = 0) } \\
\leq &
\Pr(h(x) \neq h^*(x) \mid x_A = 0) \\
= &
\frac{\Pr(h(x) \neq h^*(x), x_A = 0)}{\Pr(x_A = 0)} \\
\leq &
\frac{\Pr(h(x) \neq h^*(x), x_A = 0)}{\pminor},
\end{align*}
where the first inequality is by triangle inequality; the second inequality is by the definition of $\pminor$. Symmetrically, we have $\abs{\Pr(h(x) = +1 \mid x_A = 1) - \Pr(h^*(x) = +1 \mid x_A = 1) } \leq \frac{\Pr(h(x) \neq h^*(x), x_A = 1)}{\pminor}$. Adding up the two inequalities, we have:
\begin{align*}
& \abs{ \mu(h) - \mu(h^*) } \\
\leq &
\abs{\Pr(h(x) = +1 \mid x_A = 0) - \Pr(h^*(x) = +1 \mid x_A = 0) }
+
\abs{\Pr(h(x) = +1 \mid x_A = 1) - \Pr(h^*(x) = +1 \mid x_A = 1) } \\
\leq & \frac{\Pr(h(x) \neq h^*(x), x_A = 0)}{\pminor} + \frac{\Pr(h(x) \neq h^*(x), x_A = 1)}{\pminor} \\ 
= & 
\frac{\Pr(h(x) \neq h^*(x))}{\pminor} 
\leq 
\frac{\alpha}{\pminor}.
\qedhere
\end{align*}
\end{proof}

\section{Deferred Materials from Section~\ref{sec:mp}}

\subsection{Proof of Theorems~\ref{thm:cost-det-al-ub} and~\ref{thm:cost-det-al-lb}}
\label{sec:pf-cost-det-al}


\begin{proof}[Proof of Theorem~\ref{thm:cost-det-al-ub}]
Suppose Algorithm~\ref{alg:opt-det-auditing} (denoted as $\Acal$ throughout the proof) interacts with some target classifier $h^* \in \Hcal$. 

We will show the following claim: 
at any stage of $\Acal$, if the set of labeled examples $L$ shown so far induces a version $V = \Hcal[L]$, then $\Acal$ will subsequently query at most $\CC(V)$ more labels before exiting the while loop. 
    \item 
Note that Theorem~\ref{thm:cost-det-al-ub} follows from this claim by taking $L = \emptyset$ and $V = \Hcal$: after $\CC(\Hcal)$ label queries, it exits the while loop, which implies that,
the queried unlabeled examples $S_{\Acal,h^*}$ induces version space $V' = \Hcal(h^*, S_{\Acal,h^*})$ with 
\[
\max_{h\in V'} \mu(h) - \min_{h \in V'} \mu(h) 
= 
\diam_\mu(V')
\leq 2\epsilon.
\]
Also, note that $h^* \in V'$; this implies that 
$\mu(h^*) \in [ \min_{h \in V'} \mu(h), \max_{h \in V'} \mu(h)  ]$. Combining these two observations, we have
\[
\abs{ \hat{\mu} - \mu(h^*) }
\leq 
\frac12 \del{ \max_{h\in V'} \mu(h) - \min_{h \in V'} \mu(h) }
\leq \epsilon.
\]

We now come back to proving this claim by induction on $\CC(V)$.

\paragraph{Base case.} If $\CC(V) = 0$, then $\Acal$ immediately exits the while loop without further label queries. 

\paragraph{Inductive case.} Suppose the claim holds for all $V$ such that $\CC(V) \leq n$. Now consider a version space $V$ with $\CC(V) = n+1$. 
In this case, first recall that 
\[
\CC(V) = 1 + \min_{x \in \Xcal} \max_{y \in \cbr{-1,+1}} \CC \del{V_x^y},
\]
i.e. $\min_{x \in \Xcal} \max_{y \in \cbr{-1,+1}} \CC \del{V_x^y} = \CC(V) - 1 = n$. 
Also, recall that by the definition of Algorithm~\ref{alg:opt-det-auditing}, when facing version space $V$, the next query example $x_0$ chosen by $\Acal$ is a solution of the following minimax optimization problem: 
\[
x_0 = \argmin_{x \in \Xcal} \max_{y \in \cbr{-1,+1}} \CC \del{V_x^y},
\]
which implies that $\max_{y \in \cbr{-1,+1}} \CC \del{V_x^y} = n$.
Specifically, this implies that the version space at the next iteration, $V\del{ h^*, \cbr{x_0} } = V_{x_0}^{h^*(x_0)}$, satisfies that $\CC(V\del{ h^*, \cbr{x_0} }) \leq n$. Combining with the inductive hypothesis, we have seen that after a total of $1 + \CC(V\del{ h^*, \cbr{x_0} }) \leq n+1 = \CC(V)$ number of label queries, $\Acal$ will exit the while loop. 

This completes the inductive proof of the claim.
\end{proof}




\begin{proof}[Proof of Theorem~\ref{thm:cost-det-al-lb}]
Fix a deterministic active fairness auditing algorithm $\Acal$.
We will show the following claim: If $\Acal$ has already obtained an ordered sequence of labeled examples $L$, and has a remaining label budget $N \leq  \CC(\Hcal[L]) - 1$, then there exists $h \in \Hcal[L]$, such that, $\Acal$, when interacting with $h$ as the target classifier:
\begin{enumerate}
\item obtains a sequence of labeled examples $L$ in the first $|L|$ rounds;
\item has final version space $\Hcal(h, S_{\Acal,h})$ with $\mu$-diameter $> 2\epsilon$. 
\end{enumerate}

The theorem follow from this claim by taking $L = \emptyset$. To see why, we let $h \in \Hcal[\emptyset] = \Hcal$ be the classifier described in the claim.
First, note that there exists some other classifier $h' \neq h$ in the final version space $\Hcal(h, S_{\Acal,h})$, such that $\abs{ \mu(h') - \mu(h) } > 2\epsilon$.
For such $h'$, 
$h'(S_{\Acal,h}) = h(S_{\Acal,h})$. 
Therefore, by Lemma~\ref{lem:int-history-agree},  $S_{\Acal,h} = S_{\Acal, h'}$ (which we denote by  $S$ subsequently), and $h$ and $h'$ have the exact same labeling on $S$, and $\hat{\mu}_{\Acal,h} = \hat{\mu}_{\Acal,h'}$. This implies that, for $\Acal$, at least one of the following must be true:
\[
\abs{ \hat{\mu}_{\Acal,h} - \mu(h) } > \epsilon \text{ or }
\abs{ \hat{\mu}_{\Acal,h'} - \mu(h') } > \epsilon,
\]
showing that it does not guarantee an estimation error $\leq \epsilon$ under all target $h \in \Hcal$.

We now turn to proving the above claim by induction on $\Acal$'s remaining label budget $N$. In the following, denote by $V = \Hcal[L]$. 
\paragraph{Base case.} If $N = 0$ and $\CC(V) \geq 1$, then $\Acal$ at this point has zero label budget, which means that it is not allowed to make more queries. 
In this case, $S_{\Acal,h} = L$, and 
$\Hcal(S_{\Acal,h}, h) = V$. As $\CC(V) \geq 1$, we know that \[ 
\max_{h_1, h_2 \in \Hcal(h, S_{\Acal,h})} \abs{ \mu(h_1) - \mu(h_2) } = \max_{h_1, h_2 \in V} \abs{ \mu(h_1) - \mu(h_2) } > 2\epsilon.
\]
This completes the proof of the base case.

\paragraph{Inductive case.} Suppose the claim holds for all $N \leq n$. 
Now, suppose in the learning process, $\Acal$ has a remaining label budget $N = n+1$, and has obtained labeled examples $L$ such that $V = \Hcal[L]$ satisfies $\CC(V) \geq n+2$. Let $x$ be the next example $\Acal$ queries. By the definition of $\CC$, there  exists some $y \in \cbr{-1,+1}$, such that 
\[ 
\CC\del{ \Hcal\sbr{ L \cup \cbr{ (x,y) }} } = \CC\del{ V_x^y } \geq \CC(V) - 1 \geq n+1,
\]
and after making this query, the learner has a remaining label budget of $N - 1 = n$. 

By inductive hypothesis, there exists some $h \in \Hcal\sbr{L \cup \cbr{(x, y)}}$, such that when $\Acal$ interacts with $h$ subsequently (with obtained labeled examples $L \cup \cbr{(x, y)}$ and label budget $<n$), the final unlabeled dataset $S_{\Acal, h}$ satisfies
\[
\diam_\mu \del{ \Hcal(h, S_{\Acal,h}) }
=
\max_{h_1, h_2 \in \Hcal(h, S_{\Acal,h})} \abs{ \mu(h_1) - \mu(h_2) } 
> 2\epsilon.
\]
In addition, when interacting with $h$, $\Acal$ obtains the example sequence $\langle L, (x,y) \rangle$ in its first $|L|+1$ rounds of interaction, which implies that it obtains the  example sequence $L$ in its first $|L|$ rounds of interaction with $h$. 
This completes the induction.
\end{proof}

\subsection{Proof Sketch of Proposition~\ref{prop:cost-iid}}

\begin{proof}[Proof sketch]
Let $S_1$ and $S_2$ be $O\del{ \frac{1}{\epsilon^2} \ln |\Hcal| }$ i.i.d samples from $D_X \mid x_A = 1$ and $D_X \mid x_A = 0$, respectively.
Define 
\[
\hat{\mu}(h, S_1, S_2) 
= 
\Pr_{x \sim S_1}(h(x) = +1) - \Pr_{x \sim S_2}(h(x) = +1).
\]
Hoeffding's inequality and union bound guarantees that with probability at least $\frac12$, $\forall h \in \Hcal$, $|\hat{\mu}(h, S_1, S_2) - \mu(h)| \leq \epsilon$. 
Now consider the following deterministic algorithm $\Acal$: 
\begin{itemize}
\item Let $n = O\del{ \frac{1}{\epsilon^2} \ln |\Hcal| }$; 
\item Find (the lexicographically smallest) $S_1$ and $S_2$ in $\Xcal^n$, such that
\begin{equation}
\forall h \in \Hcal, \;\;
\abs{ \hat{\mu}(h, S_1, S_2) - \mu(h)} \leq \epsilon.
\label{eqn:derandom-iid-constr}
\end{equation}
This optimization problem is feasible, because as we have seen, a random choice of $S_1, S_2$ makes Equation~\eqref{eqn:derandom-iid-constr} happen with nonzero probability. 
\item Return $\hat{\mu}(h^*, S_1, S_2)$ with $2n$ label queries to examples in $S_1 \cup S_2$. 
\end{itemize}
By its construction, $\Acal$ queries $2n = O\del{ \frac{1}{\epsilon^2} \ln |\Hcal| }$ labels and returns $\hat{\mu}$ that is $\epsilon$-close to $\mu(h^*)$.
\end{proof}

\subsection{Proof of Proposition~\ref{prop:cost-cal}}
\label{sec:cost-cal-deferred}

Before we prove Proposition~\ref{prop:cost-cal}, 
we first recall the well-known Bernstein's inequality: 

\begin{lemma}[Bernstein's inequality]
Given a set of iid random variables $Z_1, \ldots, Z_n$ with mean $\mu$ and variance $\sigma^2$; in addition, $|Z_i| \leq b$ almost surely.
Then, 
with probability $1-\delta$,
\[
\abs{ \frac1n \sum_{i=1}^n Z_i - \mu }
\leq  
\sqrt{\frac{2 \sigma^2 \ln\frac2\delta }{n}}
+
\frac{b \ln\frac2\delta}{3n}.
\]
\end{lemma}

\begin{proof}[Proof of Proposition~\ref{prop:cost-cal}]
We will analyze Algorithm~\ref{alg:cal-derandomized}, a derandomized version of the Phased CAL algorithm~\citep[][Chapter 2]{hsu2010algorithms}. To prove this proposition, using Theorem~\ref{thm:cost-det-al-lb}, it suffices to show that Algorithm~\ref{alg:cal-derandomized} has a deterministic label complexity bound of $O\del{ \theta(\epsilon) \cdot \ln|\Hcal| \cdot \ln\frac1\epsilon }$.

We first show that for every $n$, the optimization problem in line~\ref{line:sample-select} is always feasible. To see this, observe that if we draw $S_n = \cbr{x_1, \ldots, x_{m_n}}$ as sample of size $m_n$ drawn iid from $D_X$, we have:
\begin{enumerate}
\item By Bernstein's inequality with $Z_i = I( x_i \in \DIS(V_n) )$, with probability $1-\frac14$,
\begin{align*}
\Pr_{S_n}( x \in \DIS(V_n)) 
\leq & \Pr_{D_X}( x \in \DIS(V_n) ) + \sqrt{ \frac{ 2 \Pr_{D_X}( x \in \DIS(V_n) ) \ln 8}{m_n} } + \frac{\ln 8}{3 m_n} \\
\leq &
2 \Pr_{D_X}( x \in \DIS(V_n) ) + \frac{\ln 8}{m_n}.
\end{align*}
where the second inequality uses Arithmetic Mean-Geometric Mean (AM-GM) inequality.

\item By Bernstein's inequality and union bound over $h, h' \in \Hcal$, we have with probability $1-\frac14$,
\begin{align*}
\forall h, h' \in \Hcal: \;\;\;\; \Pr_{D_X}( h(x) \neq h'(x) )
\leq 
\Pr_{S_n}( h(x) \neq h'(x) ) 
+ 
\sqrt{ \frac{ 4 \Pr_{D_X}( h(x) \neq h'(x) ) \ln |\Hcal|}{m_n} } + \frac{4 \ln |\Hcal|}{3m_n}
\end{align*}
in which, 
\begin{align*}
    \forall h, h' \in \Hcal: \;\;\;\; & \Pr_{S_n}( h(x) \neq h'(x) ) = 0 
    \implies  \Pr_{D_X}( h(x) \neq h'(x) ) \leq \frac{16 \ln|\Hcal|}{m_n}. 
\end{align*}
\end{enumerate}
By union bound, with nonzero probability, the above two condition hold simultaneously, showing the feasibility of the optimization problem.


We then argue that for all $n$, $V_{n+1} \subseteq B(h^*, \frac{16\ln|\Hcal|}{m_n})$. This is
because for all $h \in V_{n+1}$, it and $h^*$ are both in $V_n$ and therefore they agree on $S_n \setminus T_n$; on the other hand, $h$ and $h^*$ agree on $T_n$ by the definition of of $V_{n+1}$. As a consequence, 
$\Pr_{S_n}( h(x) \neq h^*(x) ) = 0$, which implies that $\Pr_{D_X}( h(x) \neq h^*(x) ) \leq \frac{16 \ln|\Hcal|}{m_n}$. 
As a consequence, for all $h \in V_{N+1}$, 
$\Pr(h(x) \neq h^*(x)) \leq 
\frac{16\ln|\Hcal|}{m_N} 
\leq 
\pminor \epsilon$, implying that $\abs{ \mu(h) - \mu(h^*) } \leq \epsilon$ (recall Lemma~\ref{lem:pac-auditing}).

We now turn to upper bounding Algorithm~\ref{alg:cal-derandomized}'s label complexity:
\begin{align*}
\sum_{n=1}^N |T_n|
= &
\sum_{n=1}^N m_n \cdot ( 2 \Pr_{D_X}( x \in \DIS(V_n) ) + \frac{\ln 8}{m_n} ) \\
\leq & \sum_{n=1}^N m_n \cdot ( \theta(\epsilon) \cdot \frac{16 \ln|\Hcal|}{m_n} \cdot \frac{2}{\pminor} + \frac{\ln 8}{m_n}   ) \\
\leq & 
O\del{ \theta(\epsilon) \cdot \ln|\Hcal| \cdot \ln\frac1\epsilon },
\end{align*}
where the inequality uses the observation that for every $n \in [N]$,
\[
\Pr_{D_X}( x \in \DIS(V_n)) \leq \Pr_{D_X}\del{ x \in \DIS(B(h^*, \frac{16\ln|\Hcal|}{m_n}))} \leq \theta(\frac{\pminor \epsilon}{2}) \cdot \frac{16\ln|\Hcal|}{m_n} \leq \theta(\epsilon) \cdot \frac{16\ln|\Hcal|}{m_n} \cdot \frac{2}{\pminor},
\]
where the second inequality is from the definition of disagreement coefficient (recall Section~\ref{sec:addl-notations}), and the 
last inequality is from a basic property of disagreement coefficient~\citep[][Corollary 7.2]{hanneke2014theory}.
\end{proof}

\subsection{Proof of Proposition~\ref{prop: costh-to-sc}}
\label{sec:pf-costh-to-sc}

We first prove the following theorem that gives a decision tree-based characterization of the $\CC(\cdot)$ function. Connections between active learning and optimal decision trees have been observed in prior works~\citep[e.g.][]{laber2004hardness,balcan2012active}.

\begin{definition}
An example-based decision tree $\Tcal$ for (instance domain, hypothesis set) pair $(\Xcal, V)$ is such that: 
\begin{enumerate}
    \item $\Tcal$'s internal nodes are examples in $\Xcal$; every internal node has two branches, with the left branch labeled as $+1$ and the right labeled as $-1$.
    \item Every leaf $l$ of $\Tcal$ corresponds to a set of classifiers $V_l \subset V$, such that all $h \in V_l$ agree with the examples that appear in the root-to-leaf path to $l$. Formally, suppose the path from the root to leaf $l$ is an alternating sequence of examples and labels $\langle x_1, y_1, \ldots, x_n, y_n \rangle$, then for every $i \in [n]$, $h(x_i) = y_i$. 
\end{enumerate} 
\end{definition}

\begin{definition}
Fix $D_X$. 
An example-based decision tree $\Tcal$ is said to $(\mu,\epsilon)$-separate a hypothesis set $V$, if
for every leaf $l$ of $\Tcal$, $V_l$ satisfies $\diam_\mu(V_l) \leq 2\epsilon$.
\end{definition}
\begin{theorem}
Given a version space $V$, $\CC(V)$ is the minimum depth of all decision trees that $(\mu,\epsilon)$-separates $V$.
\label{thm:cc-tree}
\end{theorem}
\begin{proof}
We prove the theorem by induction on $\CC(V)$.

\paragraph{Base case.} If $\CC(V) = 0$, then $\diam_\mu(V) \leq 2\epsilon$. Then there exists a trivial decision tree (with leaf only) of depth $0$ that $(\mu,\epsilon)$-separates $V$, which is also the smallest depth possible. 

\paragraph{Inductive case.} Suppose the statement holds for any $V$ such that $\CC(V) = n$. 
Now consider $V$ such that $\CC(V) = n+1$. \begin{enumerate}
    \item We first show that there exists a decision tree of depth $n+1$ that $(\mu,\epsilon)$-separates $V$. Indeed, pick $x = \argmin_{x \in \Xcal} \max_{y} \CC(V_x^y)$.
    
    With this choice of $x$, we have 
    both $\CC(V_x^{-1})$ and $\CC(V_x^{+1})$ are equal to $n$. Therefore, by inductive hypothesis for $V_x^{-1}$ and $V_x^{+1}$, we can construct decision trees $\Tcal^-$ and $\Tcal^+$ of depths $n$ that $(\mu,\epsilon)$-separate the two hypothesis classes respectively. Now define $\Tcal$ to be such that it has root node $x$, and has left subtree $\Tcal^+$ and right subtree $\Tcal^-$, we see that $\Tcal$ has depth $n+1$ and $(\mu,\epsilon)$-separates $V$. 
    
    \item We next show that any decision tree of depth $n$ does not $(\mu,\epsilon)$-separate $V$. Indeed, assume for the sake of contradiction that such tree $\Tcal$ exists. Then consider the example $x$ at the root of the tree; by the definition of $\CC$, one of $\CC(V_x^{-1})$ and $\CC(V_x^{+1})$ must be $\geq n$. Without loss of generality, assume that $V' = V_x^{+1}$ is such that $\CC( V' ) \geq n$. Therefore, there must exists some subset $V'' \subset V'$ such that $\CC(V'') = n$. Applying the inductive hypothesis on $V''$, no decision tree of depth $n-1$ can $(\mu,\epsilon)$-separate $V''$. This contradicts with the observation that 
    the left subtree of $\Tcal$, which is of depth $n-1$, $(\mu,\epsilon)$-separates $V'$. 
    \qedhere
\end{enumerate}
\end{proof}


We now restate a more precise version of Proposition~\ref{prop: costh-to-sc}. First we define the computational task of computing a $0.3\ln(|\Hcal|)$-approximation of $\CC(\Hcal)$ by the following problem:

\begin{framed}
\textbf{Problem Minimax-Cost (MC):}\\
Input: instance space $\Xcal$, hypothesis class $\Hcal$, data distribution $D_X$, precision parameter $\epsilon$. \\
Output: a number $\mcsoln$ such that $\CC(\Hcal) \leq \mcsoln \leq 0.3 \ln(|\Hcal|) \CC(\Hcal)$.
\end{framed}


\begin{proposition}[Proposition~\ref{prop: costh-to-sc} restated]
If there is an algorithm that solves Minimax-Cost in $\poly(|\Hcal|, |\Xcal|, 1/\epsilon)$ time, then $\mathrm{NP} \subseteq \mathrm{TIME}(n^{O(\log\log n)})$.
\label{prop: costh-to-sc-restated}
\end{proposition}

\begin{proof}[Proof of Proposition~\ref{prop: costh-to-sc-restated}]

Our proof takes after \cite{laber2004hardness}'s reduction from set cover (SC) to Decision Tree Problem (DTP).
Here, we reduce from SC to the Minimax-Cost problem (MC), i.e. computing $\CC(\Hcal)$ for a given hypothesis class $\Hcal$, taking into account the unique structure of active fairness auditing. 
Specifically, the following gap version of SC's decision problem has been shown to be computationally hard\footnote{The definition of Gap-SC requires that $n \geq 10$, which is without loss of generality: all Gap-SC instances with $n < 10$ are solvable in constant time.}:

\begin{framed}
\textbf{Problem Gap-Set-Cover (Gap-SC):}\\
Input: a universe $U = \{u_1, ..., u_n\}$ of size $n$ with $n \geq 10$, and a family of subsets $\Ccal = \{C_1, ..., C_m\}$, and an integer $k$, such that either of the following happens:
\begin{itemize}
    \item Case 1: $\opt_{\setc} \leq k$,
    \item Case 2: $\opt_{\setc} \geq 0.99 k \ln n$,
\end{itemize}
where $\opt_{\setc}$ denotes the minimum set cover size of $(U, \Ccal)$.

Output: 1 or 2, which case the instance is in.
\end{framed}

Specifically, it is well-known that obtaining a polynomial time algorithm for the above decision problem\footnote{The constant 0.99 can be changed to any constant $<1$~\citep{feige1998threshold}.} on minimum set cover would imply that $\mathrm{NP} \subseteq \mathrm{TIME}(n^{O(\log\log n)})$~\citep{feige1998threshold}, which is believed to be false.





To start, recall that an instance of Gap-SC problem $I_{\setc} = (U, \Ccal, k)$; an instance of the MC problem $I_{\mc} = (\Hcal, \Xcal, D_X, \epsilon)$. 


With this, we define a coarse reduction $\beta$ that constructs a MC-instance from a Gap-SC instance with universe $U=\{u_1, ..., u_n\}$ and sets $\Ccal = \{C_1, ..., C_m\}$, which will be refined shortly: 
\begin{enumerate}
\item Let $\Hcal = \cbr{ h_0, h_1, \ldots, h_n }$, where $h_0(x) \equiv -1$ always, and for all $j \in [n]$, $h_j$ corresponds to $u_j$ (the definitions of $h_j$'s will be given shortly).


\item Create example $x_0$ such that for all $h \in \Hcal$, $h(x_0) = -1$. 

\item For every $i \in [m]$, create basis example $x_i$ to correspond to $C_i$ such that for every $j \in [n]$, $h_j(x_i) = 1$ iff $u_j \in C_i$.

\item For each set $C_i$, create $|C_i|-1$ auxiliary $x$'s as follows: 
Given set $C_i$ with $|C_i|=s_i$ that corresponds to $\{h_{i1}, .., h_{is_i}\}$, create a balanced binary tree $\Tcal_i$ with each leaf corresponding to a $h_{ij}$. Create an auxiliary example associated with each internal node in $\Tcal_i$ as follows: for each internal node in the tree, define the corresponding auxiliary sample $x$ such that its label is $+1$ under all the classifiers in the leaves of the subtree rooted at its left child, and its label is $-1$ under all remaining classifiers in $\Hcal$.
The total number of auxiliary $x$'s is $\leq m \cdot (n-1)$.
\label{item:auxiliary}



\item Define $\Xcal$ as the union of the example sets constructed in the above three items, which has at most 
$N \leq mn + 1$ examples. Define $D_X$ to be such that: $x \mid x_A = 0 \sim \Uniform(\Xcal \setminus \cbr{x_0})$, and $x \mid x_A = 1 \sim \Uniform(\cbr{x_0})$, and set $\epsilon = 1/(2N)$.
With this setting of $\epsilon$, for every $h \in \Hcal$ such that $h \neq h_0$, 
$\abs{\mu(h) - \mu(h_0)} = \abs{\Pr( h(x) = +1 \mid x_A = 0 ) - \Pr( h_0(x) = +1 \mid x_A = 0 )} \geq \frac{1}{N-1} > 2\epsilon$.
\label{item:setting-eps}
\end{enumerate}

Recall that $\opt_{\setc}$ is defined as the size of an optimal solution for SC instance $(U, \Ccal)$; we let $\opt_{\mc}$ denote the height of the tree corresponding to the optimal query strategy for the MC instance $I_{\mc}$ obtained through reduction $\beta$. We have the following result:

\begin{lemma}
\label{lem:sc-mc}
$
\opt_{\setc}
\leq 
\opt_{\mc} 
\leq 
\opt_{\setc} + \max\limits_{C \in \Ccal} \log |C|
$.
\end{lemma}
\begin{proof}
Let $k = \opt_\setc$. We show the two inequalities respectively.
\begin{enumerate}
\item By Theorem~\ref{thm:cc-tree}, it suffices to show that any example-based decision tree $\Tcal$ that $(\mu, \epsilon)$-separates $\Hcal$ must have depth at least $k$. To see this, first note that by item~\ref{item:setting-eps} in the reduction $\beta$ and the definition of $(\mu, \epsilon)$-separation, the leaf in $\Tcal$ that contains $h_0$ must not contain other hypotheses in $\Hcal$. In addition, as $h_0 \equiv -1$, $h_0$ must lie in the rightmost leaf of $\Tcal$.

Now to prove the statement, we know that the examples along the rightmost path of $\Tcal$ corresponds to a collection of sets that form a set cover of $\Ccal$. It suffices to show that this set cover has size no greater than the set cover of $I_\setc$. This is because the examples along the rightmost path are either $x_i$'s, which correspond to some set in $\Ccal$, or auxiliary examples which correspond to some subset of a set in $\Ccal$. A set cover instance with $U$ and $\Ccal'$ where $\Ccal'$ comprises of sets from $\Ccal$ and subsets of sets from $\Ccal$ will not have a smaller set cover.



Therefore, the length of the path from the root to the rightmost leaf is at least $k$, the size of the smallest set cover of the original SC instance $I_{\setc}$. 



\item Let an optimal solution for $I_{\setc}$ be $G = \{i_1, ..., i_k\}$. Below, we construct an example-based decision tree $\Tcal$ of depth $k + \max\limits_{C \in \Ccal} \log |C|$ that $(\mu,\epsilon)$-separates $\Hcal$:

Let the rightmost path of $\Tcal$ contain nodes corresponding to $x_{i_1}, ..., x_{i_k}$ (the order of these are not important). At level $l=1, ..., k$, the left subtree of $x_{i_l}$ is defined to be $\Tcal_{i_l}$ as defined in step~\ref{item:auxiliary} of reduction $\beta$. Note that this may result in $\Tcal$ with potentially empty leaves, in that for some $h$ covered by multiple $x_{i_l}$'s, it only appears in $x_{i_o}$ where $o = \min\cbr{l: h(x_{i_l}) = +1}$. 


We will prove that by the above construction, $\Tcal$ $(\mu,\epsilon)$-separates $\Hcal$, as every leaf corresponds to a version space $V$ that is a singleton set (and thus has $\diam_\mu(V) = 0 \leq 2\epsilon$):
\begin{enumerate}

    \item For all but the rightmost leaf, this holds by the construction of $\Tcal_{i_l}$'s.
    
    \item For the rightmost leaf, we will show that only $h_0$ is in the version space. Since $G$ is a set cover, we have that $\cup_{l=1}^k C_{i_l} = U$. Therefore, $\forall j \in [n]$, $\exists l \in [k]$ such that $u_j \in C_{i_l} \Leftrightarrow h_j(x_{i_l}) = 1$ by construction. This implies that the all zero labeling of $x_{i_1}, ..., x_{i_k}$ can only correspond to $h_0$. Therefore, the version space at the rightmost leaf $V$ satisfies $|V|=\cbr{h_0}$.
\end{enumerate}



Recall from Theorem~\ref{thm:cc-tree} that the depth of $\Tcal$ upper bounds $\opt_{\mc}$.
$\Tcal$'s maximum root to leaf path is of length at most $k + \max_{C \in \Ccal} \log |C|$.
\qedhere
\end{enumerate}
\end{proof}




Built from $\beta$, we now construct an improved gap preserving reduction $\beta'$, defined as follows. Given any Gap-SC instance $I_{\setc} = (U, \Ccal, k)$ with 
universe $U=\{u_1, ..., u_n\}$ and sets $\Ccal = \{C_1, ..., C_m\}$: 
\begin{enumerate}
\item Take constant $z = \log n$. Construct a Gap-SC instance $I_{\setc,z} = (U^z, \Ccal^z, kz)$,  containing $z$ copies of the original set
covering instance: $U^z = \{u_1^1, \ldots ,u^1_n, \ldots ,u_1^z, \ldots , u^z_n\}$, 
$\Ccal^z = \cbr{ C_1, \ldots, C_{zm} }$, where
$C_{(p-1)m+i}=\{u^p_{i1}, \ldots, u^p_{is_i}\}$ for $p \in [z]$, $i \in [m]$. Note that $\opt_{\setc,z} = k \opt_{\setc}$.

\item Apply reduction $\beta$ to obtain $I_{\mc,z}$ from $I_{\setc,z}$.
\end{enumerate}

Now, we will argue that $\beta'$ is a gap-preserving reduction:

\begin{enumerate}
    \item Suppose the original Gap-SC instance $I_{\setc} = (U, \Ccal, k)$ is in case 1, i.e., $\opt_{\setc} \leq k$. Then, $\opt_{\setc,z} \leq kz$. By Lemma~\ref{lem:sc-mc}, $\opt_{\mc,z} \leq kz + \max_{C \in \Ccal^z} \log|C| \leq  kz + \log n \leq z(k + 1) \leq 2zk$.
    
    \item Suppose the original Gap-SC instance $I_{\setc} = (U, \Ccal, k)$ is in case 2, i.e., $\opt_{\sc} \geq 0.99 k \ln n$. Then, $\opt_{\setc,z} \geq 0.99 zk \ln n $, which by Lemma~\ref{lem:sc-mc}, yields that $\opt_{\mc,z} \geq 0.99 zk \ln n$.
\end{enumerate}

Now suppose that there exists an algorithm $\Acal$ that solves the MC problem in $\poly(|\Hcal|, |\Xcal|, \frac{1}{\epsilon})$ time. We propose the following algorithm $\Acal'$ that solves the Gap-SC problem in polynomial time, which, as mentioned above, implies that $\mathrm{NP} \subseteq \mathrm{TIME}(n^{O(\log\log n)})$:
\begin{framed}
Input: $I_{\setc} = (U, \Ccal, k)$.
\begin{itemize}
    \item Apply $\beta'$ on $I_{\setc}$ to obtain an instance of MC, $I_{\mc,z}$
    \item Let $\mcsoln \gets \Acal( I_{\mc,z} )$. Output 1 if $\mcsoln \leq 0.7 z k \ln n$, and 2 otherwise.
\end{itemize}
\end{framed}

\paragraph{Correctness.} As seen above, if $I_{\setc}$ is in case 1, then $\opt_{\mc, z} \leq 2zk$. For $n \geq 10$, by the guarantee of $\Acal$, $\mcsoln \leq 0.3 \ln|\Hcal| \cdot \opt_{\mc, z} \leq 0.6 \ln( n \log n) \cdot z k \leq 0.7 z k \ln n$, and $\Acal'$ outputs 1. Otherwise, $I_{\setc}$ is in case 2, then $\opt_{\mc, z} \geq 0.99zk \ln n$, and by the guarantee of $\Acal$, $\mcsoln \geq 0.99zk \ln n > 0.7 zk \ln n$, and $\Acal'$ outputs 2.

\paragraph{Time complexity.} In $I_{\mc,z}$, $|\Xcal| \leq (mz \cdot nz + 1) = O( mn \log^2 n)$, $|\Hcal| = n z = n \log n$, and $\epsilon = \frac{1}{2N} = \frac{1}{2(mz \cdot nz+1)} = \Omega( \frac{1}{mn \log^2 n} )$. As $\Acal$ runs in time $O(\poly(|\Xcal|, |\Hcal|, \frac{1}{\epsilon}))$, $\Acal'$ runs in time $O(\poly(m,n))$.
\end{proof}

\subsection{Deferred Materials for Section~\ref{sec:efficient-audit}}

\label{sec:hegedus}

\subsubsection{$(\mu,\epsilon)$-specifying set, $(\mu,\epsilon)$-teaching dimension and their properties}

The following definitions are inspired by the teaching and  exact active learning literature~\cite{hegedHus1995generalized,hanneke2007teaching}.

\begin{definition}[$(\mu,\epsilon)$-specifying set]
Fix hypothesis class $\Hcal$ and any function $h: \Xcal \to \Ycal$,\footnote{Note that $h$ is allowed to be outside $\Hcal$.} a set of unlabeled examples $S$ is said to be a \emph{$(\mu,\epsilon)$-specifying set} for $h$ and $\Hcal$, if $\forall h_1, h_2 \in \Hcal(h, S) \centerdot |\mu(h_1) - \mu(h_2)| \leq 2 \epsilon$.
\label{def:mu-spec-set}
\end{definition}

\begin{definition}[$(\mu,\epsilon)$-extended teaching dimension] 
Fix hypothesis class $\Hcal$ and any function $h: \Xcal \to \Ycal$, define
$t(h, \Hcal, \mu, \epsilon)$ as the size of the minimum $(\mu,\epsilon)$-specifying set for $h$ and $\Hcal$,
i.e. it is 
the optimal solution of the following optimization problem (OP-$h$):
\[
\min |S|, s.t. \forall h_1, h_2 \in \Hcal(h, S) \centerdot |\mu(h_1) - \mu(h_2)| \leq 2\epsilon
\]
\end{definition}

\begin{definition}
We define the $\mu$-extended teaching dimension $\XTD(\Hcal, \mu, \epsilon) := \max_{h: \Xcal \to \Ycal} t(h,\Hcal, \mu, \epsilon)$.
\label{def:mu-xtd}
\end{definition}

The improper teaching dimension is related to $\CC(\Hcal)$ in that:

\begin{lemma}
\label{lem:cost-td}
$$\XTD(\Hcal, \mu, \epsilon) \leq \CC(\Hcal).$$
\end{lemma}
\begin{proof}
Let $h_0 = \argmax_{h: \Xcal \to \Ycal} t(h,\Hcal, \mu, \epsilon)$. Let $k$ denote $t(h_0 ,\Hcal, \mu, \epsilon) - 1$.
It suffices to show that $\CC(\Hcal) \geq k$. 
To see this, first note that
\begin{align*}
    \CC(\Hcal) & = 1 + \min_x \max_y \CC(\Hcal[(x,y)]) \\
    & \geq 1 + \min_{x_1 \in \Xcal} \CC(\Hcal[(x, h_0(x))]) \\
    & \geq 2 + \min_{x_1 \in \Xcal} \min_{x_2 \in \Xcal} \CC(\Hcal[\{(x_1, h_0(x_1)), (x_2, h_0(x_2))\}])
\end{align*}
We can repeatedly unroll the above expression
as long as $\diam_\mu(\Hcal[\{(x_1, h_0(x_1)), \ldots, (x_i, h_0(x_i))])$ is at least $> 2\epsilon$.
After unrolling $k-1$ times where $U_{k-1} = \langle x_1, \ldots, x_{k-1} \rangle$, we have 
\[
\CC(\Hcal)
\geq 
k-1 + \min_{U_{k-1}} \CC(\Hcal(h_0, U_{k-1})).
\]
 
By the definition of $t(h,\Hcal, \mu, \epsilon)$, for any $U$ with $U \leq k - 1$, there exists $h', h'' \in \Hcal(h_0, U)$ such that $|\mu(h') - \mu(h'')| > \epsilon \Rightarrow \diam_\mu(\Hcal(h_0, U)) > \epsilon$. Thus, for any unlabeled dataset $U_{k-1}$ of size $k-1$, $\CC(\Hcal(h_0, U_{k-1})) \geq 1$.
Therefore, $\CC(\Hcal) \geq k$. 
\end{proof}

\subsubsection{Proof of Theorem~\ref{thm:auditing-complexity-oe}}

\begin{proof}

We prove the theorem as follows:


\paragraph{Correctness.} Observe that right before Algorithm~\ref{alg:efficient-audit} returns, it must execute lines~\ref{line:return-1} and~\ref{line:return-condition}. 
Since the condition on line~\ref{line:return-condition} is also satisfied, the dataset $T$ must be such that $\hat{h}(T) = h^*(T)$.
Combined with the definitions of optimization problems~\eqref{eqn:h-1} and~\eqref{eqn:h-2}, this implies that, the $h_1$ and $h_2$ used in line~\ref{line:return-1} right before return satisfy that
\[
\mu(h_1) = \min_{h \in \Hcal(h^*,T)} \mu(h),
\quad
\mu(h_2) = \max_{h \in \Hcal(h^*,T)} \mu(h).
\]
Therefore, $\mu(h^*) \in [\min_{h \in \Hcal(h^*,T)} \mu(h), \max_{h \in \Hcal(h^*,T)} \mu(h)] = [\mu(h_1), \mu(h_2)]$. 
Furthermore, by line~\ref{line:return-1}, $\mu(h_1) - \mu(h_2) \leq 2\epsilon$.
Hence, $\hat{\mu}$, the output of Algorithm~\ref{alg:efficient-audit}, satisfies that,
\[ 
\abs{ \hat{\mu} - \mu(h^*) } = \abs{ \frac12 \del{ \mu(h_1) + \mu(h_2) } - \mu(h^*) } \leq \epsilon.
\]


\paragraph{Label complexity.} We now bound the label complexity of the algorithm, specifically, in terms of $\XTD(\Hcal, \mu, \epsilon)$.

First, at the end of the $t$-th iteration of the outer loop, the newly collected dataset $T_t$ must be such that $\exists x \in T_t$ and $\hat{h}(x) \neq h^*(x)$. As $\Ocal$ has a mistake bound of $M$, the total number of outer loop iterations, denoted by $\oli$, must be most $M$. In addition, by Lemma~\ref{lem:outer-iter} given below, with probability $1-\delta/M$, $|T_t| \leq O\del{ \XTD(\Hcal, \mu, \epsilon) \cdot \log\frac{|\Hcal|M}{\delta} \log|\Xcal| }$. Therefore, by a union bound, with probability $1-\delta$, the total number of label queries made by Algorithm~\ref{alg:efficient-audit} is at most 
\[
\sum_{t=1}^{\oli} |T_t| 
\leq 
O \del{ M \cdot \XTD(\Hcal, \mu, \epsilon) \cdot \log\frac{|\Hcal|M}{\delta} \log|\Xcal|}.
\qedhere
\]
\end{proof}

\begin{lemma}
\label{lem:outer-iter}
For every outer iteration of Algorithm~\ref{alg:efficient-audit}, with probability $\geq 1-\frac{\delta}{M}$, $T$, the dataset at the end of this iteration, satisfies $\abs{T} \leq O\del{ \XTD(\Hcal, \mu, \epsilon) \cdot \log\frac{|\Hcal|M}{\delta} \log|\Xcal| }$. 
\end{lemma}
\begin{proof}

The inner loop is similar to the ``black-box teaching'' algorithm of~\cite{dasgupta2019teaching} except that we are teaching $\mu(\hat{h})$ as opposed to $\hat{h}$ itself. 
Although~\cite{dasgupta2019teaching}'s algorithm was originally designed for exact (interactive) teaching, it implicitly gives an oracle-efficient algorithm for approximately computing the minimum set cover; we will use this insight throughout the proof. As the analysis of~\cite{dasgupta2019teaching} is only on the \emph{expected} number of teaching examples, we use a different filtration to obtain a high probability bound over the number of teaching examples.

First we setup some useful notations for the proof. let $\Xcal = \cbr{x_1, \ldots, x_m}$. Recall that $\lambda = \ln \frac{|\Hcal|^2 M}{\delta}$. 
Let $W_i(x)$ denote the weight of point $x \in \Xcal$ (denoted by $w(x)$ in the algorithm) at the end of round $i$ of the inner loop and let $\thres_{x_j}$ be the  exponentially-distributed threshold associated with $x_j$. Define random variable $U_{i,j} = \ind\{\thres_{x_j} > W_{i}(x_j)\}$. Let $M_i$ denotes the number of teaching examples selected in the $i$th round of doubling; it can be seen that $M_i = \sum_{j \in [m]} U_{i,j}$.
Also define $(i,j) \preceq (i',j')$ iff $(i,j)$ precedes  $(i',j')$ lexicographically.

Define two filtrations:
\begin{enumerate}

\item Let $\Fcal_{i,j}$ be the sigma-field of all indicator events $\{ U_{i',j'}: (i',j') \preceq (i,j) \}$. As a convention, $\Fcal_{i,0} := \Fcal_{i-1, m}$. 

\item Let $\Fcal_i$ be the sigma-field of all indicator events $\{U_{i',j'} :j' \in [m], 1\leq i' \leq i\}$; this is the filtration used by~\cite{dasgupta2019teaching}. It can be easily seen that $\Fcal_i = \Fcal_{i,m}$. 
\end{enumerate}

Define $Y_{i,j} = \sum_{ (i',j') \preceq (i, j) } Z_{i',j'}$, where
$ Z_{i,j} = U_{i,j} - \EE\sbr{U_{i,j} \mid \Fcal_{i,j-1} } \in [-1,+1]$. Then $Y_{i,j}$ is a martingale as $\EE[Y_{i,j} | \Fcal_{i,j - 1}] = \EE[Z_{i,j} | \Fcal_{i,j - 1}] + \EE[Y_{i,j-1} | \Fcal_{i,j - 1}] = Y_{i,j-1}$.

Let $N$ 
be the total number of rounds, which by item 1 of Lemma~\ref{lem:num-iters}, is 
 $O(\XTD(\Hcal, \mu, \epsilon) \ln |\Xcal|)$ (Lemma 4 of \cite{dasgupta2019teaching}) with probability 1. 
 We may then apply Freedman's inequality (Lemma~\ref{lem:freedman}): 
 since $Y_{i,j} - Y_{i, j-1} = Z_{i,j} \leq 1$ almost surely, 
 for any $s$ and any $\sigma^2 > 0$, 
 
\begin{equation}
\Pr\del{ \exists n, m, Y_{nm} \geq s, \sum_{(i,j) \preceq (n,m)} \EE[Z_{ij}^2 | \Fcal_{i(j-1)}] \leq \sigma^2 }
\leq 
\exp\del{ -\frac{s^2}{2(\sigma^2 + s / 3)}}
\label{eqn:freedman}
\end{equation}


Next, we let $\sigma^2 = \lambda (1 + \XTD(\Hcal, \mu, \epsilon) \ln(2|\Xcal|))$; 
we have for any $n,m$:
\begin{align*}
   & \sum_{(i,j) \preceq (n,m)} \EE[Z_{ij}^2 | \Fcal_{i(j-1)}] \\
    & = \sum_{(i,j) \preceq (n,m)} \EE[U_{ij}^2 | \Fcal_{i(j-1)}] - \EE[U_{ij} | \Fcal_{i(j-1)}]^2 \\
    & \leq \sum_{(i,j) \preceq (n,m)} \EE[U_{ij}^2 | \Fcal_{i(j-1)}] \\
    & = \sum_{(i,j) \preceq (n,m)} \EE[U_{ij} | \Fcal_{i(j-1)}] \\
    & = \sum_{i=1}^n \EE_{\Fcal_{i-1}}\sbr{M_i } \\
    & \leq \lambda \sum_{x \in \Xcal} W_n(x) \tag{Lemma~\ref{lem:weight-diff}} \\
    & \leq \lambda (1 + \XTD(\Hcal, \mu, \epsilon) \ln(2|\Xcal|)) = \sigma^2. \tag{Lemma~\ref{lem:num-iters}}
\end{align*}

Meanwhile, we choose $s = \frac16 \log (\frac{1}{\delta}) + \sqrt{2\sigma^2 \log\frac1\delta  +  \frac16 \log (\frac{1}{\delta})} = O\del{ \sqrt{\ln \frac1\delta} \sigma + \ln \frac1\delta }$, which ensures that the right hand side of Eq.~\eqref{eqn:freedman} is at most $\delta$.


Thus, by Equation~\eqref{eqn:freedman}, we have with probability $1-\delta$, for all $n,m$, 
$$Y_{nm} = \sum_{(i',j') \preceq (n, m)} U_{i'j'} - \sum_{i=1}^n \EE_{\Fcal_{i-1}}\sbr{M_i } \leq  O\del{ \sqrt{\ln \frac1\delta} \sigma + \ln \frac1\delta }.
$$

Also, using Lemma~\ref{lem:weight-diff} and~\ref{lem:num-iters}, with probability $1$, $\sum_{i=1}^N \EE_{\Fcal_{i-1}}\sbr{M_i } \leq \lambda (1 + \XTD(\Hcal, \mu, \epsilon) \ln(2|\Xcal|))$.

Therefore, for $Y_{Nm}$ in particular, 
\begin{align*}
    Y_{Nm} &\leq O\del{ \lambda (1 + \XTD(\Hcal, \mu, \epsilon) \ln(2|\Xcal|)) + \sqrt{ \lambda (1 + \XTD(\Hcal, \mu, \epsilon) \ln(2|\Xcal|)) \ln (1/\delta)} + \ln (1/\delta) }\\
    & = O\del{ \lambda (1 + \XTD(\Hcal, \mu, \epsilon) \ln(2|\Xcal|)) + \ln\frac1\delta } \\
    & = O\del{ \XTD(\Hcal, \mu, \epsilon) \ln(|\Xcal|) \ln((|\Hcal|M)/\delta) }.
    \qedhere
\end{align*}
\end{proof}

\begin{lemma}[Freedman's Inequality]
\label{lem:freedman}
Let martingale $\{Y_k\}_{k=0}^{\infty}$ with difference sequence $\{X_k\}_{k=0}^{\infty}$ be such that $X_k \leq R$ a.s for all $k$ and $Y_0 = 0$. Let $W_k = \sum_{j=1}^k \EE_{j-1}[X_j^2]$. Then, for all $t \geq 0$ and $\sigma^2 > 0$:
$$\Pr(\exists k \geq 0: Y_k \geq t \land W_k \leq \sigma^2) \leq \exp\left( -\frac{t^2 / 2}{\sigma^2 + Rt / 3} \right).$$
\end{lemma}

\begin{lemma}
\label{lem:num-iters}
For any outer iteration of Algorithm~\ref{alg:efficient-audit}:
\begin{enumerate}
    \item The number of inner loop iterations is at most $\XTD(\Hcal, \mu, \epsilon) \cdot \log(2|\Xcal|)$.
    \item At any point in the inner loop, we have that, $\sum_{x \in \Xcal} w(x) \leq 1 + \XTD(\Hcal, \mu, \epsilon) \cdot \log(2|\Xcal|)$.
\end{enumerate}
\end{lemma}
\begin{proof}
The proof is very similar to~\citet[][Lemma 4]{dasgupta2019teaching} with some differences; for completeness, we include a proof here. 

We first prove the second item.
First, note that at any point of the algorithm, for all $x$, $w(x) \leq 2$. Let $S^*(\hat{h})$ be the optimal solution of optimization problem (OP-$\hat{h}$) - we have $|S^*(\hat{h})| = t(\hat{h}, \Hcal, \mu, \epsilon) \leq \XTD(\Hcal, \mu, \epsilon)$. Note that every time when line~\ref{line:doubling} is called, by the feasibility of $S^*(\hat{h})$ with respect to (OP-$\hat{h}$), 
$\Delta(h_1, h_2) \cap S^*(\hat{h}) \neq \emptyset$, therefore,
the weight of some element $x \in S^*(\hat{h})$ gets doubled. 
This implies that the total number of times line~\ref{line:doubling} is executed is at most $|S^*(\hat{h})| \cdot \log(2|\Xcal|)$.  
Otherwise, if the number of time line~\ref{line:doubling} is executed is $\geq |S^*(\hat{h})| \cdot \log(2|\Xcal|) + 1$, by the pigeonhole principle, there must exist some element $x \in S^*(\hat{h})$ whose weight exceeds $1$, which is a contradiction. 

Finally, note that each weight doubling only increases the total weight by $\leq 1$, we have the final total weight is at most 
\[
1 + 1 \cdot |S^*(\hat{h})| \cdot \log(2|\Xcal|)
\leq 
1 + \XTD(\Hcal, \mu, \epsilon) \cdot \log(2|\Xcal|).
\]

The first item follows since the number of inner iterations is at most the number of weight doublings. 
\end{proof}

\begin{lemma}
For every inner iteration, 
$\EE[M_i | \Fcal_{i-1}] \leq \sum_{x \in \Xcal} \lambda (W_{i}(x) - W_{i-1}(x))$.
\label{lem:weight-diff}
\end{lemma}
\begin{proof}
The proof is almost a verbatim copy of~\citet[][Lemma 6]{dasgupta2019teaching}, which we include here:
\begin{align*}
    \EE[M_i | \Fcal_{i-1}] & = \sum_{x \in \Xcal} \Pr(x\text{ chosen in round }i | x\text{ not chosen before round }i, \Fcal_{i-1}) \\
    & = \sum_{x \in \Xcal} 1 - \Pr(\thres_x > W_i(x) | \tau_x > W_{i-1}(x)) \\
    & = \sum_{x \in \Xcal} (1 - \exp(-\lambda (W_i(x) - W_{i-1}(x)))) \\
    & \leq \sum_{x \in \Xcal} \lambda (W_i(x) - W_{i-1}(x)).
    \qedhere
\end{align*}
\end{proof}



\section{Deferred Materials from Section~\ref{sec:estimation}}


\subsection{Distribution-free Query Complexity Lower Bounds for Auditing with VC classes}
\label{sec:audit-vc}


\begin{theorem}[Lower bound for randomized auditing]
\label{thm:audit-lb-vc}
If hypothesis class $\Hcal$ has VC dimension $d \geq 1600$, 
and $\epsilon \in (0, \frac1{40}]$, 
then 
for any (possibly randomized) algorithm $\Acal$, there exists a distribution $D$ realizable by $h^* \in \Hcal$, such that when $\Acal$ is given a querying budget $N \leq \Omega( \min(d, \frac1{\epsilon^2}) )$, its output $\hat{\mu}$ is such that 
\[
\PP\del{ \abs{ \hat{\mu} - \mu(h^*) } > \epsilon } > \frac18.
\]
\end{theorem}

\begin{proof}
We will be using Le Cam's method with several subtle modifications. First, we will reduce the estimation problem to a hypothesis testing problem, where under different hypotheses, the $\mu(h^*)$ will be centered around two $\Omega(\epsilon)$-separated values with high probability. 
Second, we will upper bound the distribution divergence of the interaction history under the two hypotheses; this requires some delicate handling, as the label on a queried example depends not only on the identity of the example, but also historical labeled examples. 

\paragraph{Step 1: the construction.} As $\VC(\Hcal) = d$, there exists a set of examples $Z = \cbr{z_0, z_1, \ldots, z_{d-1}} \subset \Xcal$ shattered by $\Hcal$. 
Let $Z_+ = \cbr{z_1, \ldots, z_{d-1}}$. 
Let $D_X$ be as follows: $x \mid x_A = 0$ is uniform over $Z_+$, whereas $x \mid x_A = 1$ is the delta mass on $z_0$.

Let $\tilde{\epsilon} = 10 \max(\epsilon, \frac{1}{\sqrt{d}})$; by the conditions that $d \geq 1600$ and $\epsilon \leq \frac1{40}$, we have $\tilde{\epsilon} \leq \frac14$.
Let label budget $N = \frac{1}{24 \tilde{\epsilon}^2} =  \Omega\del{ \min(d, \frac1{\epsilon^2} ) }$. 

Consider two hypotheses that choose $h^*$ randomly from $\cbr{-1,+1}^{Z_+}$, subject to $h^*(z_0) = 0$: 
\begin{itemize}
    \item $H_0$: choose $h^*$ such that for every $i \in [d-1]$, independently, 
    $h^*(z_i) = \begin{cases} +1, & \text{ with probability } \frac12 -\tilde{\epsilon} \\
    -1, & \text{ with probability } \frac12 + \tilde{\epsilon}
    \end{cases}$
    
    \item $H_1$: choose $h^*$ such that for every $i \in [d-1]$, independently, 
    $h^*(z_i) = \begin{cases} +1, & \text{ with probability } \frac12 + \tilde{\epsilon} \\
    -1, & \text{ with probability } \frac12 - \tilde{\epsilon}
    \end{cases}$
\end{itemize}

We have the following simple claim that shows the separation of $\mu(h^*)$ under the two hypotheses. Its proof is deferred to the end of the main proof. 
\begin{claim}
$
\PP_{h^* \sim H_0} \del{ \mu(h^*) \leq \frac12 - \frac12 \tilde{\epsilon} } \geq \frac{15}{16}
$,
and
$
\PP_{h^* \sim H_1} \del{ \mu(h^*) \geq \frac12 + \frac12 \tilde{\epsilon} } \geq \frac{15}{16}
$.
\label{claim:hyp-mu-prob-vc}
\end{claim}

\paragraph{Step 2: upper bounding the statistical distance.} Next, we show that $H_0$ and $H_1$ are hard to distinguish with $\Acal$ having a label budget of $N$. To this end, we upper bound the KL divergence of the joint distributions of $\langle (x_1, y_1), \ldots, (x_n, y_n) \rangle =: (x, y)_{\leq n}$ under $H_0$ and $H_1$, denoted as $\PP_0$ and $\PP_1$ respectively. Applying Lemma~\ref{lem:div-decomp-al}, we have:
\begin{align}
\KL( \PP_0, \PP_1 )
= & \sum_{i=1}^n \EE \sbr{ \KL\del{ \PP_0(y_i = \cdot \mid (x,y)_{\leq i-1}, x_i)), \PP_1(y_i = \cdot \mid (x,y)_{\leq i-1}, x_i)} }.
\label{eqn:div-decomp-vc}
\end{align}
We claim that for every $i$ and $((x,y)_{\leq i-1}, x_i) \in (\Xcal \times \Ycal)^{i-1} \times \Xcal$ on the support of $\PP_0$, 
\begin{equation} 
\KL \del{ \PP_0(y_i = \cdot \mid (x,y)_{\leq i-1}, x_i)), \PP_1(y_i = \cdot \mid (x,y)_{\leq i-1}, x_i) } \leq 3\tilde{\epsilon}^2.
\label{eqn:kl-step-i-vc}
\end{equation}

First, observe that if $\langle (x,y)_{\leq i-1}, x_i \rangle$ is in the support of $\PP_0$, there must exists some $h^*: Z \to \cbr{-1, +1}$ such that $h^*(x_j) = y_j$ for all $j \in [i-1]$; in particular, this means there must not exist $j_1 \neq j_2$ in $[i-1]$, such that $x_{j_1} = x_{j_2}$ but $y_{j_1} \neq y_{j_2}$.

Next, we note that, under $H_0$, conditioned on $(x,y)_{\leq i-1}$, the posterior distribution of $h^*$ is supported over the set $\cbr{ h \mid h: Z \to \cbr{-1,+1}, \forall j \in [i-1], h(x_j) = y_j}$, and specifically, for all $x \in Z \setminus \cbr{ x_j: j \in [i-1] }$, the $h^*(x)$'s are independent conditioned on $(x,y)_{\leq i-1}$, and 
\[
\PP_0\del{ h^*(x) = +1 \mid (x,y)_{\leq i-1} } 
= \frac12 - \tilde{\epsilon}.
\]
The same statement holds for $H_1$ except that for all $x \in Z \setminus \cbr{ x_j: j \in [i-1] }$, we now have
$\PP_1(h^*(x) = +1 \mid (x,y)_{\leq i-1}) 
= \frac12 + \tilde{\epsilon}$.
In addition, the conditional distribution of $y_i \mid (x,y)_{\leq i-1}, x_i$, equals the conditional distribution of $h^*(x_i) \mid (x,y)_{\leq i-1}$, under both $H_0$ and $H_1$.
We now perform a case analysis:

\begin{enumerate}
\item If $x_i \in \cbr{ x_j: j \in [i-1] }$, then under both $H_0$ and $H_1$, the distributions of $h^*(x_i) \mid (x,y)_{\leq i-1}$ are equal: they both equal to the delta mass supported on the only element of the singleton set $\cbr{y_j: j \in [i-1], x_j = x_i}$. In this case, $\KL\del{ \PP_0(y_i = \cdot \mid (x,y)_{\leq i-1}, x_i)), \PP_1(y_i = \cdot \mid (x,y)_{\leq i-1}, x_i) } = 0 \leq 3 \tilde{\epsilon}^2$. 

\item Otherwise, $x_i \notin \cbr{ x_j: j \in [i-1] }$. Under $H_0$, $h^*(x_i) \mid (x,y)_{\leq i-1}$ takes value $+1$ with probability $\frac12 - \tilde{\epsilon}$, and takes value $-1$ with probability $\frac12 + \tilde{\epsilon}$; similarly, under $H_1$, $h^*(x_i) \mid (x,y)_{\leq i-1}$ takes value $+1$ with probability $\frac12 + \tilde{\epsilon}$, and takes value $-1$ with probability $\frac12 - \tilde{\epsilon}$.  In this case, by Fact~\ref{fact:kl-ber} and that $\tilde{\epsilon} \leq \frac14$, $\KL\del{ \PP_0(y_i = \cdot \mid (x,y)_{\leq i-1}, x_i)), \PP_1(y_i = \cdot \mid (x,y)_{\leq i-1}, x_i) } = \kl\del{ \frac12-\tilde{\epsilon}, \frac12 + \tilde{\epsilon} } \leq 3 \tilde{\epsilon}^2$. 
\end{enumerate}

In summary, in both cases, Equation~\eqref{eqn:kl-step-i-vc} holds, and plugging this back to Equation~\eqref{eqn:div-decomp-vc} with $n = \frac{1}{24 \tilde{\epsilon}^2}$, we have $\KL(\PP_0, \PP_1) \leq 3 n \tilde{\epsilon}^2 \leq \frac{1}{8}$. By Pinsker's inequality (Lemma~\ref{lem:pinsker}), $d_{\TV}(\PP_0, \PP_1) \leq \sqrt{ \frac12 \KL(\PP_0, \PP_1) } \leq \frac12$. By Le Cam's Lemma (Lemma~\ref{lem:lecam}), for any hypothesis tester $\hat{b}$,  
we have
\begin{equation}
\frac12 \PP_0 \del{ \hat{b} = 1 }
+ 
\frac12 \PP_1 \del{ \hat{b} = 0 }
\geq 
\frac12 \del{ 1 - d_{\TV}(\PP_0, \PP_1) }
\geq \frac{1}{4}.
\label{eqn:bayes-error-lb-vc}
\end{equation}

\paragraph{Step 3: concluding the proof.} Given $\Acal$'s output auditing estimate $\hat{\mu}$, consider the following hypothesis test: 
\[ 
\hat{b} = \begin{cases} 0, & \hat{\mu} < \frac12, \\ 1, & \hat{\mu} \geq \frac12.  \end{cases}
\]
Plugging into Equation~\eqref{eqn:bayes-error-lb-vc}, we have
\begin{equation}
\frac12 \PP_0 \del{ \hat{\mu} \geq \frac12 }
+ 
\frac12 \PP_1 \del{ \hat{\mu} < \frac12 }
\geq 
\frac 1 {4}.
\label{eqn:hat-b-error-vc}
\end{equation}
Now, recall  Claim~\ref{claim:hyp-mu-prob-vc}, and using the fact that 
$\PP(A \cap B) \geq \PP(A) - \PP(B^C) = \PP(A) + \PP(B) - 1$, we have
\begin{equation}
\PP_0 \del{ \abs{\hat{\mu} - \mu(h^*)} \geq \frac12 \tilde{\epsilon} }
\geq 
\PP_0 \del{ \hat{\mu} \geq \frac12, \mu(h^*) \leq \frac12 - \frac12 \tilde{\epsilon} }
\geq 
\PP_0 \del{ \hat{\mu} \geq \frac12 } + \frac{15}{16} - 1 \geq \PP_0 \del{ \hat{\mu} \geq \frac12 } - \frac{1}{16}.
\label{eqn:p0-1over16-vc}
\end{equation} 
Symmetrically, we also have 
\begin{equation}
\PP_1 \del{ \abs{\hat{\mu} - \mu(h^*)} \geq \frac12 \tilde{\epsilon} }
\geq 
\PP_1 \del{ \hat{\mu} < \frac12, \mu(h^*) \geq \frac12 + \frac12 \tilde{\epsilon} } \geq \PP_1 \del{ \hat{\mu} < \frac12 } - \frac{1}{16}.
    \label{eqn:p1-1over16-vc}
\end{equation}

Combining Equations~\eqref{eqn:hat-b-error-vc},~\eqref{eqn:p0-1over16-vc}, and~\eqref{eqn:p1-1over16-vc}, we have 
\[
\frac12 \PP_0 \del{ \abs{\hat{\mu} - \mu(h^*)} \geq \frac12 \tilde{\epsilon} }
+ 
\frac12 \PP_1 \del{ \abs{\hat{\mu} - \mu(h^*)} \geq \frac12 \tilde{\epsilon} }
\geq \frac{1}{4} - \frac{1}{16} > \frac{1}{8}. 
\]
As $\frac12 \tilde{\epsilon} > \epsilon$, and 
the left hand side can be viewed as the total probability of $\abs{\hat{\mu} - \mu(h^*)} > \epsilon$ when $h^*$ is drawn from the uniform mixture distribution of the $h^*$ distributions under $H_0$ and $H_1$. By the probabilistic method, there exists some $h^*$ such that
$\PP_{h^*, \Acal} \del{ \abs{\hat{\mu} - \mu(h^*)} > \epsilon } > \frac18$. 
\end{proof}

\begin{proof}[Proof of Claim~\ref{claim:hyp-mu-prob-vc}]
Without loss of generality, we show the first inequality; the second inequality can be shown symmetrically. 
Note that under $H_0$, the random $h^*$'s DP value satisfies  
\[ 
\mu(h^*) 
= 
\Pr(h^*(x) = +1 \mid x_A = 0) - \Pr(h^*(x) = +1 \mid x_A = 1)
= 
\frac1{d-1} \sum_{i=1}^{d-1} \ind\{h^*(z_i) = +1\}
,
\]
where the second equality follows from that $\Pr(h^*(x) = +1 \mid x_A = 1) = 0$ as $h^*(z_0) = -1$ is always true. 

Under $H_0$, $(d-1) \mu(h^*)$ is the sum of $(d-1)$ iid Bernoulli random variables with mean parameter $\frac12 - \tilde{\epsilon}$. 
Therefore,
by Hoeffding's inequality, we have \[
\PP_0 \del{ \mu(h^*) > \frac12 - \frac12 \tilde{\epsilon} }
\leq 
\exp \del{ -2 (d-1) \cdot \del{ \frac12 \tilde{\epsilon} }^2 }
\leq 
\frac{1}{16},
\]
where the second inequality uses the fact that $\tilde{\epsilon} = 10 \max\del{ \epsilon, \frac{1}{\sqrt{d}} } \geq \frac{10}{\sqrt{d}}$. 
\end{proof}

\subsection{Query Complexity for Auditing Non-homogeneous Halfspaces under Gaussian Subpopulations}
\label{sec:halfspace}

\begin{theorem}[Lower bound]
\label{thm:lb-halfspace}
Let $d \geq 6400$ and $\epsilon \in (0, \frac{1}{80}]$. 
If $D_X$ is such that $x \mid x_A = 0 \sim \N(0_d, I_d)$, whereas $x \mid x_A = 1 \sim \N(0_d, (0)_{d \times d})$ (i.e. the delta-mass supported at $0_d$). For any (possibly randomized) algorithm $\Acal$, there exists $h^*$ in $\Hcal_{\hs}$ the class of nonhomogeneous linear classifiers, such that when $\Acal$ is given a query budget $N \leq \Omega\del{ \min(d, \frac1{\epsilon^2}) }$, its output $\hat{\mu}$ is such that 
\[
\PP_{\Acal, h^*}\del{ \abs{ \hat{\mu} - \mu(h^*) } > \epsilon } > \frac18.
\]
\end{theorem}

\begin{proof}

Similar to the proof of Theorem~\ref{thm:audit-lb-vc}, we will use Le Cam's method. In addition to the same challenges in the proof of Theorem~\ref{thm:audit-lb-vc}, in the active fairness auditing for halfspaces setting, we are faced with the extra challenge that the posterior distributions of $h^*(x_i) \mid (x, y)_{\leq i-1}$ deviates significantly from the prior distribution of $h^*(x_i)$, and cannot be easily calculated in closed form. To get around this difficulty, using the chain rule of KL divergence, along with 
the posterior formula for noiseless Bayesian linear regression with Gaussian prior, we calculate a tight upper bound on the KL divergence between two carefully constructed, well-separated hypotheses. 

\paragraph{Step 1: the construction.}
Let $\tilde{\epsilon} = 40 \max( \epsilon, \frac{1}{\sqrt{d}} )$; by the assumption that $\epsilon \leq \frac{1}{80}$ and $d \geq 6400$, we have $\tilde{\epsilon} \leq \frac12$.
Let label budget $N = \frac{1}{64 \tilde{\epsilon}^2} = \Omega\del{ \min(d, \frac1{\epsilon^2}) }$. 

Consider two hypotheses that choose $h^* = h_{a^*,b^*}$, such that $b^* = -1$, and $a^*$ is chosen randomly from different distributions:
\begin{itemize}
    \item $H_0: a^* \sim \N(0, \frac{1}{d} (1 + \tilde{\epsilon}) I_d)$
    \item $H_1: a^* \sim \N(0, \frac{1}{d} (1 - \tilde{\epsilon}) I_d )$
\end{itemize}
We have the following claim that shows the separation of $\mu(h^*)$ under the two hypotheses. Its proof is deferred to the end of the main proof. 
\begin{claim}
$
\PP_{h^* \sim H_0} \del{ \mu(h^*) > \Phi(-1) + \frac{\tilde{\epsilon}}{36} } \geq \frac{15}{16}
$,
and
$
\PP_{h^* \sim H_1} \del{ \mu(h^*) < \Phi(-1) - \frac{\tilde{\epsilon}}{36} } \geq \frac{15}{16}
$,
where $\Phi(z) = \int_{-\infty}^z \frac{1}{\sqrt{2\pi}} e^{-\frac{z^2}{2}} \dif z$ is the standard normal CDF.
\label{claim:hyp-mu-prob}
\end{claim}

\paragraph{Step 2: upper bounding the statistical distance.} Next, we show that $H_0$ and $H_1$ are hard to distinguish with $\Acal$ making $n \leq N$ label queries. To this end, we upper bound the KL divergence of the joint distributions of $(x,y)_{\leq n}$ under $H_0$ and $H_1$, denoted as $\PP_0$ and $\PP_1$ respectively. 
To this end, define $\tilde{y}_i = \inner{a^*}{x_i} - 1$ for $i \in [n]$, and $y_i = \sign(\tilde{y}_i)$.
Define $\tilde{\PP}_0$ and $\tilde{\PP}_1$ (resp. $\QQ_0$ and $\QQ_1$) as the joint distributions of $(x, \tilde{y})_{\leq n}$ (resp. $(x, y,  \tilde{y})_{\leq n}$) under $H_0$ and $H_1$ respectively. 
By the chain rule of KL divergence (Lemma~\ref{lem:chain-kl} with $Z = (x,y)_{\leq n}, W = \tilde{y}_{\leq n}$ and $Z = (x,\tilde{y})_{\leq n}, W = y_{\leq n}$ respectively), we get:
\begin{align*}
& \KL( \QQ_0( (x, y, \tilde{y})_{\leq n} ), \QQ_1( (x, y, \tilde{y})_{\leq n} )
\\
= &
\underbrace{\KL( \QQ_0( (x, y)_{\leq n} ), \QQ_1( (x, y)_{\leq n} )}_{\KL( \PP_0, \PP_1 )}
+
\underbrace{\KL( \QQ_0( (\tilde{y})_{\leq n} \mid (x, y)_{\leq n} ) , \QQ_1( (\tilde{y})_{\leq n} \mid (x, y)_{\leq n} ) )}_{\geq 0} \\
= & 
\underbrace{\KL( \QQ_0( (x, \tilde{y})_{\leq n} ), \QQ_1( (x, \tilde{y})_{\leq n} )}_{\KL( \tilde{\PP}_0, \tilde{\PP}_1 )}
+
\underbrace{
\KL( \QQ_0( (y)_{\leq n} \mid (x, \tilde{y})_{\leq n} ) , \QQ_1( (y)_{\leq n} \mid (x, \tilde{y})_{\leq n} ) )}_{0},
\end{align*}
where the last term is 0 because under both $\QQ_0$ and $\QQ_1$, $(y)_{\leq n} \mid (x, \tilde{y})_{\leq n}$ is the delta mass supported on $(\sign(\tilde{y}))_{\leq n}$. 
As a consequence,
\[
\KL(\PP_0, \PP_1) \leq \KL( \tilde{\PP}_0, \tilde{\PP}_1 )
\]
Also, note that $\Acal$ can be viewed as a query learning algorithm that at round $i$, receives $(x, \tilde{y})_{\leq i-1}$ as input, and choose the next example for query (i.e., it elects to only use the thresholded value $y_j$'s as opposed to the $\tilde{y}_j$'s). 
Applying Lemma~\ref{lem:div-decomp-al}, we have:
\begin{align}
\KL( \tilde{\PP}_0, \tilde{\PP}_1 )
= & \sum_{i=1}^n \EE \sbr{ \KL( \PP_0(\tilde{y}_i = \cdot \mid (x,\tilde{y})_{\leq i-1}, x_i)), \PP_1(\tilde{y}_i = \cdot \mid (x,\tilde{y})_{\leq i-1}, x_i)) }.
\label{eqn:div-decomp-hs}
\end{align}
We claim that for every $i$ and $((x,\tilde{y})_{\leq i-1}, x_i) \in (\Xcal \times \Ycal)^{i-1} \times \Xcal$ on the support of $\tilde{\PP}_0$, 
\begin{equation} 
\KL( \PP_0(\tilde{y}_i = \cdot \mid (x,\tilde{y})_{\leq i-1}, x_i)), \PP_1(\tilde{y}_i  = \cdot \mid (x,\tilde{y})_{\leq i-1}, x_i)) \leq 3\tilde{\epsilon}^2.
\label{eqn:kl-step-i-hs}
\end{equation}

First, by Lemma~\ref{lem:active-posterior} (deferred to the end of the proof), under $H_0$, conditioned on $(x,\tilde{y})_{\leq i-1}$ on the support of $\tilde{\PP}_0$, the posterior distribution of $a^*$ is the same as $a^* \sim \N(0, \frac{1}{d} (1 + \tilde{\epsilon}) I_d)$ conditioned on the affine set $S = \cbr{ a \in \RR^d: \inner{a}{x_l} + 1 = \tilde{y}_l, \forall l \in [i-1] }$. Denote $X_{i-1} = [x_1^\top; x_2^\top; \ldots, x_{i-1}^\top] \in \RR^{(i-1) \times d}$, and $\tilde{Y}_{i-1} = (\tilde{y}_1, \ldots, \tilde{y}_{i-1})$;
for $(x,\tilde{y})_{\leq i-1}$ on the support of $\tilde{\PP}_0$, it must be the case that $S \neq \emptyset$, and as a result, 
$\hat{a} = X_{i-1}^\dagger (\tilde{Y}_{i-1} - \one_{i-1}) \in S$. Also, denote by $X_{i-1}^{\perp}$ a matrix whose columns are an orthonormal basis of $\spn(x_1, \ldots, x_{i-1})$; such a $X_{i-1}^{\perp}$ is always well-defined as $i-1 \leq n-1 \leq d-1$.
Applying Lemma~\ref{lem:normal-affine}, we have

\[ 
a^* \mid (x, \tilde{y})_{\leq i-1} \sim \N \del{ \hat{a}, \frac{1}{d} (1+\tilde{\epsilon}) X_{i-1}^{\perp} (X_{i-1}^{\perp})^\top },
\]
with its covariance matrix $\frac{1}{d} (1+\tilde{\epsilon}) X_{i-1}^{\perp} (X_{i-1}^{\perp})^\top$ being rank-deficient. 

Now, observe that $\tilde{y}_i \mid (x,\tilde{y})_{\leq i-1}, x_i$ has the same distribution as $\inner{a^*}{x_i} + 1 \mid (x, \tilde{y})_{\leq i-1}$, which is $\N\del{ \inner{\hat{a}}{x_i} + 1, \frac{1}{d} (1+\tilde{\epsilon}) x_i^\top X_{i-1}^{\perp} (X_{i-1}^{\perp})^\top x_i }$. 

Similarly, under $H_1$, we have $\tilde{y}_i \mid (x,\tilde{y})_{\leq i-1}, x_i$ has distribution $\N\del{ \inner{\hat{a}}{x_i} + 1, \frac{1}{d} (1 - \tilde{\epsilon}) x_i^\top X_{i-1}^{\perp} (X_{i-1}^{\perp})^\top x_i }$. 
We now prove~\eqref{eqn:kl-step-i-hs} by a case analysis:
\begin{enumerate}
    \item If $x_i \in \spn( x_1, \ldots, x_{i-1} )$, then $(X_{i-1}^{\perp})^\top x_i = 0$, and under both $H_0$ and $H_1$, the posterior distributions of $\tilde{y}_i \mid (x,\tilde{y})_{\leq i-1}, x_i$ are both delta mass on $\inner{\hat{a}}{x_i} + 1$, and therefore, $\KL( \PP_0(\tilde{y}_i = \cdot \mid (x,\tilde{y})_{\leq i-1}, x_i)), \PP_1(\tilde{y}_i = \cdot \mid (x,\tilde{y})_{\leq i-1}, x_i)) = 0 \leq 3\tilde{\epsilon}^2$.
    
    \item If $x_i \notin \spn( x_1, \ldots, x_{i-1})$, then $(X_{i-1}^{\perp})^\top x_i \neq 0$, and under $H_0$ and $H_1$, the posterior distributions of $\tilde{y}_i \mid (x,\tilde{y})_{\leq i-1}, x_i$ are $\N(\hat{\mu}_i, (1+\tilde{\epsilon}) \sigma_i^2 )$ and  $\N(\hat{\mu}_i, (1-\tilde{\epsilon}) \sigma_i^2 )$ respectively, where $\hat{\mu}_i = \inner{\hat{a}}{x_i} + 1$, and $\sigma_i^2 = \frac{1}{d} x_i^\top X_{i-1}^{\perp} (X_{i-1}^{\perp})^\top x_i$. In this case, by Fact~\ref{fact:kl-normal}, 
    \begin{align*}
    & \KL\del{ \PP_0(\tilde{y}_i = \cdot \mid (x,\tilde{y})_{\leq i-1}, x_i)), \PP_1(\tilde{y}_i = \cdot \mid (x,\tilde{y})_{\leq i-1}, x_i)} \\
    = & 
    \KL \del{ \N(\hat{\mu}_i, (1+\tilde{\epsilon}) \sigma_i^2 ), \N(\hat{\mu}_i, (1-\tilde{\epsilon}) \sigma_i^2 ) }
    \\
    = & 
    \frac{1}{2} \del{ \frac{1+\tilde{\epsilon}}{1-\tilde{\epsilon}} - 1 + \ln(\frac{1-\tilde{\epsilon}}{1+\tilde{\epsilon}}) }
    \\
    \leq & \frac12 \del{\frac{2 \tilde{\epsilon}}{1-\tilde{\epsilon}}}^2
    \\
    \leq & 8 \tilde{\epsilon}^2,
    \end{align*}
\end{enumerate}
where the first inequality is by the fact that $\ln(1+x) \geq x - x^2$ when $x \geq 0$, and taking $x = \frac{2 \tilde{\epsilon}}{1-\tilde{\epsilon}}$, and the second inequality is from $\tilde{\epsilon} \leq \frac12$ and  algebra.




In summary, in both cases, Equation~\eqref{eqn:kl-step-i-hs} holds, and plugging this back to Equation~\eqref{eqn:div-decomp-hs} with $n \leq \frac{1}{64 \tilde{\epsilon}^2}$, we have $\KL(\PP_0, \PP_1) \leq 8 n \tilde{\epsilon}^2 \leq \frac{1}{8}$. By Pinsker's inequality (Lemma~\ref{lem:pinsker}), $d_{\TV}(\PP_0, \PP_1) \leq \sqrt{ \frac12 \KL(\PP_0, \PP_1) } \leq \frac12$. Le Cam's lemma (Lemma~\ref{lem:lecam}) implies that, for any hypothesis tester $\hat{b}$, 
we have
\begin{equation}
\frac12 \PP_0( \hat{b} = 1 )
+ 
\frac12 \PP_1( \hat{b} = 0 )
=
\frac12 (1 - d_{\TV}(\PP_0, \PP_1)) 
\geq \frac{1}{4}.
\label{eqn:bayes-error-lb-hs}
\end{equation}

\paragraph{Step 3: concluding the proof.} Given $\Acal$'s output auditing estimate $\hat{\mu}$, consider the following hypothesis tester: 
\[ 
\hat{b} = \begin{cases} 0, & \hat{\mu} > \Phi(-1), \\ 1, & \hat{\mu} \leq \Phi(-1).  \end{cases}
\]
Plugging into Equation~\eqref{eqn:bayes-error-lb-vc}, we have
\begin{equation}
\frac12 \PP_0\del{ \hat{\mu} \leq \Phi(-1) }
+ 
\frac12 \PP_1\del{ \hat{\mu} > \Phi(-1) }
\geq 
\frac 1 {4}.
\label{eqn:hat-b-error}
\end{equation}
Now, recall  Claim~\ref{claim:hyp-mu-prob}, and using the fact that 
$\PP(A \cap B) \geq \PP(A) - \PP(B^C) = \PP(A) + \PP(B) - 1$, we have
\begin{equation}
\PP_0 \del{ \abs{\hat{\mu} - \mu(h^*)} \geq \frac1{36} \tilde{\epsilon} }
\geq 
\PP_0 \del{ \hat{\mu} \leq \Phi(-1), \mu(h^*) > \Phi(1) - \frac1{36} \tilde{\epsilon} } 
\geq 
\PP_0 \del{ \hat{\mu} \leq \Phi(-1) } + \frac{15}{16} - 1 
\geq 
\PP_0\del{ \hat{\mu} \leq \Phi(-1) } - \frac{1}{16}.
\label{eqn:p0-1over16}
\end{equation} 
Symmetrically, we also have 
\begin{equation}
\PP_1 \del{ \abs{\hat{\mu} - \mu(h^*)} \geq \frac1{36} \tilde{\epsilon} }
\geq 
\PP_1 \del{ \hat{\mu} > \Phi(-1) } - \frac{1}{16}.
    \label{eqn:p1-1over16}
\end{equation}
Combining Equations~\eqref{eqn:hat-b-error},~\eqref{eqn:p0-1over16}, and~\eqref{eqn:p1-1over16}, we have 
\[
\frac12 \PP_0 \del{ \abs{\hat{\mu} - \mu(h^*)} \geq \frac1{36} \tilde{\epsilon} } 
+ 
\frac12 \PP_1 \del{ \abs{\hat{\mu} - \mu(h^*)} \geq \frac1{36} \tilde{\epsilon} }
\geq \frac{1}{4} - \frac{1}{16} > \frac{1}{8}. 
\]
As $\frac1{36} \tilde{\epsilon} \geq \epsilon$, and 
the left hand side can be viewed as the total probability of $\abs{\hat{\mu} - \mu(h^*)} \geq \epsilon$ when $h^*$ is drawn from the uniform mixture distribution of the $h^*$ distributions under $H_0$ and $H_1$. By the probabilistic method,
there exists some $h^* \in \Hcal$ such that
$\PP_{h^*} \del{ \abs{\hat{\mu} - \mu(h^*)} > \epsilon } > \frac18$. 
\end{proof}

\begin{lemma}
\label{lem:active-posterior}
Given the same setting above. For any fixed $i \in \NN$ and $(x, \tilde{y})_{\leq i}$, the posterior distribution 
$a^* \mid (x, \tilde{y})_{\leq i}$ is the same as $a^* \mid \cbr{a^* \in U}$, where $U = \cbr{a: \forall j \in [i]: \inner{x_j}{a} + 1 = \tilde{y}_j }$. 
\end{lemma}
\begin{proof}
We use the Bayes formula to expand the posterior; below $\propto$ denotes equality up to a multiplicative factor independent of $a^*$. 
\begin{align*}
\PP ( a^* \mid (x, \tilde{y})_{\leq i}  )
\propto & \PP( a^*, (x, \tilde{y})_{\leq i}  ) \\
\propto & \PP( a^*) \prod_{j=1}^i \PP (x_j \mid a^*, (x, \tilde{y})_{\leq j-1} ) \PP(\tilde{y}_j \mid x_j, a^*, (x, \tilde{y})_{\leq j-1} ) \\
\propto  & \PP( a^*) \prod_{j=1}^i \PP (x_j \mid (x, \tilde{y})_{\leq j-1} ) \one\cbr{\tilde{y}_j = \inner{x_j}{a^*}+1} \\
\propto & 
\PP( a^*) \prod_{j=1}^i \one\cbr{\tilde{y}_j = \inner{x_j}{a^*}+1}
\end{align*}
where the second equality uses the definition of conditional probability; the third equality uses the fact that for any fixed query learning algorithm $\Acal$, $x_j$ is independent of $a^*$ conditioned on $(x, \tilde{y})_{\leq j-1}$, and the observation that given $x_j$ and $a^*$, $\tilde{y}_j = \inner{x_j}{a^*} + 1$ deterministically. This concludes the proof.
\end{proof}

\begin{proof}[Proof of Claim~\ref{claim:hyp-mu-prob}]
For $h^*(x) = \sign(\inner{a^*}{x} + b^*)$ where $b^* = -1$, it can be seen that,

\[
\PP_0( h^*(x) = +1 \mid x_A = 1 ) = 0,
\]
On the other hand,
\[
\PP_0( h^*(x) = +1 \mid x_A = 0) 
=
\PP_{z \sim \N(0, I_d)}( \inner{a^*}{z} \geq 1 )
= 
\PP_{z \sim \N(0, I_d)} \del{ \left\langle \frac{a^*}{\|a^*\|},{z} \right\rangle \geq \frac{1}{\| a^*\|}}
= 
1 - \Phi\del{ \frac{1}{\| a^* \|} }. 
\]

Also, note that under $H_0$, $\frac{d \| a^* \|_2^2}{(1 + \tilde{\epsilon})} \sim \chi^2(d)$; Therefore, by Fact~\ref{fact:chi-sq-conc}, we have that with probability $\geq \frac{15}{16}$, 
$
\frac{d \| a^* \|_2^2}{(1 + \tilde{\epsilon})}
\geq 
d \cdot (1 - 10 \sqrt{ \frac{1}{d} })
$,
which implies that 
\[
\frac{1}{\| a^* \|}
\leq 
\sqrt{ \frac{1}{(1 + \tilde{\epsilon}) (1 - 10 \sqrt{\frac 1 d}) }  }
\leq 
\sqrt{ \frac{1}{(1 + \tilde{\epsilon}) (1 - \frac{\tilde{\epsilon}}{4}) }  }
\leq 
1 - \frac{\tilde{\epsilon}}{4}.
\]
Therefore, as for every $a,b \in [\frac34, 1]$, $\abs{\Phi(a) - \Phi(b)} \geq \min_{\xi \in [\frac34,1]} \Phi'(\xi) \abs{a - b} \geq \frac{1}{9} \abs{a - b}$, we have:
\[
1 - \Phi\del{ \frac{1}{\| a^* \|} }
\geq 
1 - \Phi\del{ 1- \frac{\tilde{\epsilon}}{4} }
\geq 
1 - (\Phi(1) - \frac{\tilde{\epsilon}}{36})
\geq 
\Phi(-1) + \frac{\tilde{\epsilon}}{36}. 
\]
This concludes the proof of the first inequality. The second inequality is proved symmetrically. 
\end{proof}

We now present our (deterministic) active fairness auditing algorithm, Algorithm~\ref{alg:audit-nonhom-full} and its guarantees. 
Algorithm~\ref{alg:audit-nonhom-full} works under the setting when the two subpopulations are Gaussian, whose mean and covariance parameters $(m_0, \Sigma_0)$, $(m_1, \Sigma_1)$ are known. It also assumes access to black-box queries to $h^* \in \Hcal_{\hs} = \cbr{ h_{a,b}(x):= \sign(\inner{a}{x} + b): a \in \RR^d, b \in \RR }$, and aims to estimate $\mu(h^*)$ within precision $\epsilon$. Recall that 
\[
\mu(h^*) = \Pr_{x \sim D_X}\del{ h^*(x) = 1 \mid x_A = 0 } - \Pr_{x \sim D_X}\del{ h^*(x) = 1 \mid x_A = 1 },
\]
it suffices to estimate $\gamma_b := \Pr_{x \sim D_X}\del{ h^*(x) = 1 \mid x_A = 0 }$ within precision $\epsilon/2$, for each $b \in \cbr{0,1}$. To this end, we note that 
\[
\gamma_b 
= 
\Pr_{x \sim \N(m_b, \Sigma_b)}\del{ h^*(x) = 1 }
= 
\Pr_{\tilde{x} \sim \N(0, I_d)}\del{ h^*(m_b + \Sigma_b^{1/2} \tilde{x}) = 1 };
\]
if we define $\tilde{h}_b: \RR^d \to \cbr{-1,+1}$ such that 
\begin{equation}
    \tilde{h}_b ( \tilde{x} ) = h^*( m_b + \Sigma_b^{1/2} \tilde{x} ),
    \label{eqn:tilde_h_b}
\end{equation} 
$\gamma_b$ equals to 
$\gamma(\tilde{h}_b)$, where $\gamma(h) = \PP_{\tilde{x} \sim \N(0,I_d)}\del{ h(\tilde{x}) = 1 }$
is the probability of positive prediction of $h$ under the standard Gaussian distribution. Importantly, as $h^*$ is a linear classifier, $\tilde{h}_b$ is also a linear classifier and lies in $\Hcal_{\hs}$. 

Recall that procedure \estimg  (Algorithm~\ref{alg:audit-nonhom}) label-efficiently estimates $\gamma(h)$ for any $h \in \Hcal_{\hs}$, using query access to $h$. Algorithm~\ref{alg:audit-nonhom-full} uses it as a subprocedure to estimate $\gamma_b = \gamma(\tilde{h}_b)$ (line~\ref{line:estimate_gamma_b}). To simulate label queries to $\tilde{h}_b$ using query access to $h^*$, according to Equation~\eqref{eqn:tilde_h_b}, it suffices to apply an affine transformation on the input $\tilde{x}$, obtaining transformed input $m_b + \Sigma_b^{1/2} \tilde{x}$, and query $h^*$ on the transformed input. 

Finally, after $\hat{\gamma}_0, \hat{\gamma}_1$, $\epsilon/2$-accurate estimators of $\gamma_0$, $\gamma_1$ are obtained, Algorithm~\ref{alg:audit-nonhom-full} takes their difference as our estimator $\hat{\mu}$ for $\mu(h^*)$ (line~\ref{line:estimate_mu}). 

\begin{algorithm}[htb!]
\caption{Active fairness auditing for nonhomogeneous linear classifiers under Gaussian subpopulations}
\label{alg:audit-nonhom-full}
\begin{algorithmic}[1]
\REQUIRE{Subpopulation parameters $(m_0, \Sigma_0)$, $(m_1, \Sigma_1)$, query access to $h^* \in \Hcal_{\hs}$, 
target error $\epsilon$.}
\ENSURE{$\hat{\mu}$ such that $\abs{\hat{\mu} - \mu(h^*)} \leq \epsilon$.}
\FOR{$b \in \cbr{0,1}$}
\STATE Define $\tilde{h}_b: \RR^d \to \cbr{-1,+1}$ such that $\tilde{h}_b ( \tilde{x} ) = h^*( m_b + \Sigma_b^{1/2} \tilde{x} )$; 
\label{line:tilde_b}
\COMMENT{$\tilde{h}_b \in \Hcal_{\hs}$, and 
each query to $\tilde{h}_b$ can be simulated by one query to $h^*$}
\STATE $\hat{\gamma}_b \gets \estimg(\tilde{h}_b, \frac{\epsilon}{2})$
\label{line:estimate_gamma_b}
\ENDFOR
\RETURN $\hat{\gamma}_0 - \hat{\gamma}_1$
\label{line:estimate_mu}
\end{algorithmic}
\end{algorithm}

\begin{theorem}[Upper bound]
\label{thm:ub-halfspace}
If $h^* \in \Hcal_{\hs}$, $D_X$ is such that $x \mid x_A = 0 \sim \N( m_0, \Sigma_0 )$, $x \mid x_A = 1 \sim \N( m_1, \Sigma_1 )$.
Algorithm~\ref{alg:audit-nonhom-full} outputs $\hat{\mu}$, such that with probability $1$, $\abs{ \hat{\mu} - \mu(h^*) } \leq \epsilon$; moreover, Algorithm~\ref{alg:audit-nonhom-full} makes at most $O(d \ln\frac{d}{\epsilon})$ label queries to $h^*$.
\end{theorem}

\begin{proof}
As we will see from Lemma~\ref{lem:gamma-est}, for $b \in \cbr{0,1}$, the respective calls of \estimg ensures that 
\[ 
\abs{\hat{\gamma_b} - \gamma_b} \leq \frac{\epsilon}{2}. 
\]
Therefore, 
\[
\abs{ \hat{\mu} - \mu(h^*) }
\leq 
\abs{\hat{\gamma_0} - \gamma_0}
+ 
\abs{\hat{\gamma_1} - \gamma_1}
\leq 
\epsilon.
\]
Moreover, for every $b$, Lemma~\ref{lem:gamma-est} ensures that each call to \estimg only makes at most $O(d \ln\frac{d}{\epsilon})$ label queries to $\tilde{h}_b$; as simulating each query to $\tilde{h}_b$ takes one query to $h^*$, for every $b$, it also makes at most $O(d \ln\frac{d}{\epsilon})$ label queries to $h^*$. 
Summing the number of label queries over $b \in \cbr{0,1}$, the total number of label queries by Algorithm~\ref{alg:audit-nonhom-full} is $O(d \ln\frac{d}{\epsilon})$.
\end{proof}




We now turn to presenting the guarantee of the key subprocedure \estimg and its proof. This expands the analysis sketch in Section~\ref{sec:halfspace-main}.
\begin{lemma}[Guarantees of \estimg]
\label{lem:gamma-est}
Recall that $\gamma(h) = \Pr_{x \sim \N(0, I_d)}( h(x) = +1 )$. \estimg (Algorithm~\ref{alg:audit-nonhom}) receives inputs query access to $h^* \in \Hcal_{\hs}$, and target error $\epsilon$, and outputs $\hat{\gamma}$ such that 
\begin{equation}
 \abs{\hat{\gamma} - \gamma(h^*)} \leq \epsilon. 
 \label{eqn:gamma-est}
\end{equation}
Furthermore, it makes at most $O(d \ln\frac{d}{\epsilon})$ queries to $h^*$.
\end{lemma}

\begin{proof}
Let $h^*(x) = \sign(\inner{a^*}{x} + b^*)$ be the target classifier.  First, observe that $\gamma(h^*) = \Phi\del{ \frac{b^*}{\| a^* \|_2} } =:
\Phi(sr)$,
where 
$\Phi$ is the standard normal CDF, $s := \sign(b^*)$, and $r := \sqrt{\frac{1}{\sum_{i=1}^d m_i^{-2}}}$, for
$m_i := -\frac{b^*}{a_i^*}$.
Note that line~\ref{line:query-zero} of \estimg correctly obtains $s$, as $s = h^*(\vec{0}) = \sign(\inner{a^*}{\vec{0}}+b) = \sign(b)$. 

Recall that $\alpha = \sqrt{2d \ln\frac1\epsilon}$ and $\beta = 2 d^{\frac52} (\ln\frac{1}{\epsilon})^{\frac 34} (\frac1\epsilon)^{\frac12}$.
We consider two cases depending on the line in which \estimg returns:
\begin{enumerate}
    \item If \estimg returns in line~\ref{line:query-alpha-end}, then it must be the case that for all $i \in [d]$, $h^*(\alpha e_i)=h^*(-\alpha e_i)$. In this case, by Lemma~\ref{lem:same-sign}, we have that for every $i$, $\abs{m_i} \geq \alpha$. This implies that  
    $r = \sqrt{\frac{1}{\sum_{i=1}^d m_i^{-2}}} \geq \sqrt{ \frac{1}{d \alpha^{-2}} } \geq \sqrt{2 \ln\frac{1}{\epsilon}}$. For the case that $s = -1$, we have that $\gamma(h^*) = \Phi(s r) \leq \epsilon$, where we use the standard fact that $\Phi(x) \leq \exp(-\frac{x^2}{2})$ for $x \leq 0$; in this case $\hat{\gamma} = 0$ ensures Equation~\eqref{eqn:gamma-est} holds; for the symmetric case that $s = +1$, $\gamma(h^*) = \Phi(s r) \geq 1-\epsilon$ and $\hat{\gamma} = 1$, which also ensures Equation~\eqref{eqn:gamma-est}.
    
    \item On the other hand, \estimg returns in line~\ref{line:return-final}, it must be the case that there exists some $i_0 \in [d]$, such that $|m_{i_0}| \leq \alpha$. This implies that 
    $r =   \sqrt{\frac{1}{\sum_{i=1}^d m_i^{-2}}} \leq \sqrt{\frac{1}{m_{i_0}^{-2}}} = \abs{m_{i_0}} \leq \alpha$.
    
    Now, \estimg must execute lines~\ref{line:s-m-est-start} to~\ref{line:s-m-est-end}. The final $S$ it computes has the following properties:
    for every $i \in S$ added, by the guarantee of procedure \bs (Algorithm~\ref{alg:binary-search}), $\abs{\hat{m}_i - m_i} \leq \epsilon$; otherwise, for $i \notin S$, it must be the case that $h^*(\beta e_i) \neq h^*(-\beta e_i)$, which, by Lemma~\ref{lem:same-sign}, implies that $\abs{m_i} \geq \beta$. Therefore, all the conditions of Lemma~\ref{lem:hat-r-r} are satisfied, and thus, $\abs{\hat{r} - r} \leq 2\epsilon$. This also yields that $\abs{s\hat{r} - sr} \leq 2\epsilon$.
    Finally, note that $\Phi$ is $\frac{1}{\sqrt{2\pi}}$-Lipschitz, we have 
    \[
    \abs{ \hat{\gamma} - \gamma(h^*) }
    =
    \abs{ \Phi(s\hat{r}) - \Phi(sr) }
    \leq 
    \frac{1}{\sqrt{2\pi}} \cdot \abs{s\hat{r} - sr}
    \leq 
    \epsilon.
    \]
\end{enumerate}
In summary, in both cases, \estimg outputs $\hat{\gamma}$ such that Equation~\eqref{eqn:gamma-est} is satisfied.

We now calculate the total query complexity of \estimg. Line~\ref{line:query-zero} makes 1 label query; line~\ref{line:query-alpha-start} makes $2d$ label queries; for each $i \in [d]$, line~\ref{line:query-beta} makes 2 label queries, and \bs makes $\log \frac{2\beta}{\epsilon}$ label queries. In summary, the total label query complexity of \estimg is:
\[
1 + 2d + d( 2 + \log \frac{2\beta}{\epsilon} )
= 
O\del{ d \ln\frac{d}{\epsilon} }.
\qedhere
\]
\end{proof}

We now present the proof of Lemma~\ref{lem:hat-r-r}, which is key to the proof of Lemma~\ref{lem:gamma-est}.
\begin{proof}[Proof of Lemma~\ref{lem:hat-r-r}]
First, by Lemma~\ref{lem:f-lip}, and the assumption that for all $i \in S$, $\abs{\hat{m}_i - m_i} \leq \epsilon$, we have
\[
\abs{ 
\sqrt{\frac{1}{\sum_{i \in S} \hat{m}_i^{-2}} } 
-
\sqrt{\frac{1}{\sum_{i \in S} m_i^{-2}} }
}
\leq 
\epsilon.
\]
It remains to prove that 
\[
\abs{ 
\sqrt{\frac{1}{\sum_{i \in S} m_i^{-2}} }
-
\sqrt{\frac{1}{\sum_{i=1}^d m_i^{-2}} } 
}
\leq 
\epsilon,
\]
which combined with the above inequality, will conclude the proof.

To see this, let $z = \sum_{i=1}^d m_i^{-2}$ and $z_S = \sum_{i \in S} m_i^{-2}$; since for all $i \notin S$, $\abs{m_i} \geq \beta$, this implies that 
\[ 
\abs{ z - z_S } \leq \frac{d}{\beta^2}
\leq 
\frac{2 \epsilon}{ (4 d \ln\frac{1}{\epsilon})^{\frac{3}{2}} },
\]

Also, note that $\sqrt{\frac{1}{\sum_{i=1}^d m_i^{-2}}} = r \leq \alpha$ implies that $z \geq \frac{1}{\alpha^2} = \frac{1}{2d \ln\frac1\epsilon}$; therefore, $z_S \geq z - \frac{2 \epsilon}{ (4 d \ln\frac{1}{\epsilon})^{\frac{3}{2}} } \geq \frac{1}{4d \ln\frac1\epsilon}$. Now, by Lagrange mean value theorem,
\[
\abs{ \frac{1}{\sqrt{z_S}} - \frac{1}{\sqrt{z}} }
\leq 
\max_{z' \in (z_S, z)} \frac12 (z')^{-\frac32} \cdot \abs{ z_s - z }
\leq 
\frac12 (z_S)^{-\frac32} \cdot \abs{ z_s - z }
\leq 
\frac12 (4 d \ln\frac{1}{\epsilon})^{\frac{3}{2}}
\cdot 
\frac{2 \epsilon}{ (4 d \ln\frac{1}{\epsilon})^{\frac{3}{2}} }
\leq \epsilon.
\]
This concludes the proof.
\end{proof}
\begin{lemma}
\label{lem:f-lip}
Let $l \in \NN_+$ and
$f(m_1, \ldots, m_l) := \sqrt{ \frac{1}{ \sum_{i=1}^l m_i^{-2} } }$; then $f$ is 1-Lipschitz with respect to $\| \cdot \|_\infty$.
\end{lemma}

\begin{proof}
First, we show that $f$ is 1-Lipschitz with respect to $\| \cdot \|_\infty$ in each of the orthants of $\RR^l$. 
Without loss of generality, we focus on the positive orthant $R =: \cbr{m \in \RR^l: m_i \geq 0, \forall i}$. 
We now check that for any two points $\vec{m}$ and $\vec{n}$ in $R$, 
$\abs{ f(\vec{m}) - f(\vec{n}) } \leq \| \vec{m} - \vec{n} \|_\infty$. 
By Lagrange mean value theorem,
there exists some $\theta \in \cbr{t \vec{m} + (1-t) \vec{n}: t \in (0,1)}$, such that 
\[ 
\abs{ f(\vec{m}) - f(\vec{n})} = \abs{\inner{\nabla f(\theta)}{\vec{m} - \vec{n}}} \leq \| \nabla f(\theta) \|_1 \| \vec{m} - \vec{n} \|_\infty,
\]
where the second inequality is from H\"{o}lder's inequalty. Therefore, it suffices to check that for all $\vec{m}$ in the $R_0 =: \cbr{\vec{m} \in \RR^l: m_i > 0, \forall i}$ (interior of $R$), $\| \nabla f(m_1, \ldots, m_l) \|_1 \leq 1$. To see this, note that
\[
\nabla f(m_1, \ldots, m_d)
=
\del{ \frac{m_1^{-3}}{ (\sum_{i=1}^l m_i^{-2} )^{\frac{3}{2}} }, \ldots, \frac{m_l^{-3}}{ (\sum_{i=1}^l m_i^{-2} )^{\frac{3}{2}} } }
=: g,
\]
Observe that  
$\sum_{i=1}^l \abs{g_i}^{\frac{2}{3}} = 1$; this implies that for every $i \in [l]$, $\abs{g_i} \leq 1$, and therefore, 
\[
\| g \|_1 = \sum_{i=1}^l \abs{g_i} \leq 1.
\]

Now consider $\vec{m}, \vec{n} \in \RR^l$ that do not necessarily lie in the same orthant. Suppose the line segment $\cbr{ t \vec{m} + (1-t) \vec{n} : t \in [0,1] }$ consists of $k$ pieces, where piece $i$ is $\cbr{ t \vec{m} + (1-t) \vec{n} : t \in [t_{i-1}, t_{i}] }$, where $1 = t_0 > t_1 > \ldots > t_k =0$, where each piece is contained in an orthant. Then we have:
\begin{align*}
\abs{ f( \vec{m} ) - f( \vec{n} ) }
\leq & 
\sum_{i=1}^k
\abs{ f( t_{i-1} \vec{m} + (1-t_{i-1}) \vec{n} ) - f( t_i \vec{m} + (1-t_i) \vec{n} ) } \\
\leq & 
\sum_{i=1}^k
\| (t_{i-1} \vec{m} + (1-t_{i-1}) \vec{n} ) - (t_i \vec{m} + (1-t_i) \vec{n} ) \|_\infty \\
= &
\sum_{i=1}^k (t_{i-1} - t_i) \| \vec{m} - \vec{n} \|_\infty \\
= & \| \vec{m} - \vec{n} \|_\infty,
\end{align*}
where the second inequality uses the Lipchitzness of $f$ within the orthant that contains piece $i$, for each $i$ in $[k]$.
\end{proof}

\begin{lemma}
Given $i \in [d]$ and $\xi > 0$, if $h^*(\xi e_i) = h^*(-\xi e_i)$, then $\abs{m_i} \geq \xi$.
\label{lem:same-sign}
\end{lemma}
\begin{proof}
Suppose $h^*(\xi e_i) = h^*(-\xi e_i) = +1$; in this case, $-b_i \leq \xi a_i^* \leq b_i$, and therefore, $\abs{\xi a_i^*} \leq b_i$, which implies that $\abs{m_i} \geq \xi$.
The case of $h^*(\xi e_i) = h^*(-\xi e_i) = +1$ can be proved symmetrically. 
\end{proof}

\subsection{Auxiliary Lemmas for Query Learning Lower Bounds}

In this subsection we collect a few standard and useful lemmas for establishing lower bounds for general adaptive sampling and query learning algorithms, including active fairness auditing algorithms. Throughout, denote by $\PP$ the distribution of interaction transcript (the sequence of $N$ labeled examples $\langle (x_1, y_1), \ldots, (x_N, y_N) \rangle$) obtained by the query learning algorithm by interacting with the environment, and use the shorthand $(x,y)_{\leq i}$ to denote $\langle (x_1, y_1), \ldots, (x_i, y_i) \rangle$. 

\begin{lemma}[Le Cam's Lemma]
\label{lem:lecam}
Given two distributions $\PP_0$, $\PP_1$ over observation space $z \in \Zcal$, and let $\hat{b}: \Zcal \to \cbr{0,1}$ be any hypothesis tester. Then,
\begin{equation}
\frac12 \PP_0 \del{ \hat{b}(Z) = 1 }
+ 
\frac12 \PP_1 \del{ \hat{b}(Z) = 0 }
\geq 
\frac12 \del{ 1 - d_{\TV}(\PP_0, \PP_1) },
\nonumber
\end{equation}
where $d_{\TV}(\PP_0, \PP_1)$ denotes the total variation distance between $\PP_0$ and $\PP_1$. 
\end{lemma}

\begin{lemma}[Pinsker's Inequality]
\label{lem:pinsker}
For two distributions $\PP$ and $\QQ$, 
$d_{\TV}(\PP_0, \PP_1) \leq \sqrt{\frac12 \KL(\PP, \QQ)}$.
\end{lemma}

\begin{lemma}[Chain rule of KL divergence]
\label{lem:chain-kl}
For two distributions $\QQ^0(Z, W)$ and $\QQ^1(Z, W)$ over $\Zcal \times \Wcal$, we have
\begin{align*}
\KL(\QQ^0, \QQ^1) 
= & \KL( \QQ_Z^0, \QQ_Z^1 ) + \EE_{z \sim \QQ_Z^0} \sbr{ \KL( \QQ_{W \mid Z}^0(\cdot \mid z), \QQ_{W \mid Z}^1(\cdot \mid z) ) }. 
\end{align*}
\end{lemma}

\begin{fact}
Let $\kl(\cdot, \cdot)$ denote the binary relative entropy function. 
For $a,b \in [\frac14, \frac34]$, $\kl(a,b) \leq 3(b-a)^2$. 
\label{fact:kl-ber}
\end{fact}

The following lemma is well-known. 
\begin{lemma}[Divergence decomposition]
\label{lem:div-decomp-al}
For a (possibly randomized) query learning algorithm $\Acal$ with label budget $N$, under two hypotheses $H_0$, $H_1$ (represented by  distributions over the target concept $h^*$), we have:
\begin{align*}
\KL( \PP_0, \PP_1 )
= & \sum_{i=1}^N \EE \sbr{ \KL( \PP_0(y_i = \cdot \mid (x,y)_{\leq i-1}, x_i)), \PP_1(y_i = \cdot \mid (x,y)_{\leq i-1}, x_i)) }
\end{align*}
\end{lemma}
\begin{proof}
We simplify $\KL( \PP_0, \PP_1 )$ as follows:
\begin{align*}
\KL( \PP_0, \PP_1 )
= & \sum_{(x,y)_{\leq N}} \PP_0( (x,y)_{\leq N} ) \ln\frac{ \PP_0( (x,y)_{\leq N}) }{ \PP_0( (x,y)_{\leq N} ) } \\
= & \sum_{(x,y)_{\leq N}} \PP_0( (x,y)_{\leq N} ) \sum_{i=1}^N \ln\frac{\PP_\Acal(x_i \mid (x,y)_{\leq i-1})}{\PP_\Acal(x_i \mid (x,y)_{\leq i-1})} + \ln\frac{ \PP_0(y_i \mid (x,y)_{\leq i-1}, x_i) }{\PP_1(y_i \mid (x,y)_{\leq i-1}, x_i)} \\
= & \sum_{i=1}^N \sum_{(x,y)_{\leq i}} \PP_0( (x,y)_{\leq i} ) \ln\frac{ \PP_0(y_i \mid (x,y)_{\leq i-1}, x_i) }{\PP_1(y_i \mid (x,y)_{\leq i-1}, x_i)} \\
= & \sum_{i=1}^N \sum_{(x,y)_{\leq i-1}, x_i} \PP_0( (x,y)_{\leq i-1}, x_i ) \cdot  \sum_{y_i} \PP_0(y_i \mid (x,y)_{\leq i-1}, x_i)  \ln\frac{ \PP_0(y_i \mid (x,y)_{\leq i-1}, x_i) }{\PP_1(y_i \mid (x,y)_{\leq i-1}, x_i)} \\
= & \sum_{i=1}^N \EE \sbr{ \KL( \PP_0(y_i = \cdot \mid (x,y)_{\leq i-1}, x_i)), \PP_1(y_i = \cdot \mid (x,y)_{\leq i-1}, x_i)) },
\end{align*}
where the first equality is by the definition of KL divergence;
the second equality is from the chain rule of conditional probability; the third equality is by canceling out the conditional probabilities of unlabeled examples given history, as we run the same algorithm $\Acal$ under two environments; the fourth equality is by the law of total probability; the fifth equality is again by the definition of the KL divergence.   
\end{proof}

\begin{fact}[KL divergence between Gaussians of the same mean]
\label{fact:kl-normal}
If $\mu \in \RR$ and $\sigma_1, \sigma_2 > 0$, then, 
\[
\KL\del{ \N(\mu, \sigma_1^2), \N(\mu, \sigma_2^2)) } = \frac{\sigma_1^2}{\sigma_2^2} - 1 + \ln\frac{\sigma_2^2}{\sigma_1^2}.
\]
\end{fact}

\begin{fact}[Concentration of $\chi^2$ random variables]
\label{fact:chi-sq-conc}
For $d \geq 1$, $Z \sim \chi^2(d)$, and $\delta > 0$,
\[ \PP\del{ \abs{Z - d} \leq 2\sqrt{d \ln\frac1\delta} + 2 \ln\frac1\delta } \geq 1-\delta. \]
Specifically, 
\[ \PP\del{ \abs{Z - d} \leq 10 \sqrt{d} } \geq \frac{15}{16}. \]
\end{fact}

The lemma below is a standard fact on normal distribution conditioned on affine subspaces; we include a proof here as we cannot find a reference. 

\begin{lemma}
\label{lem:normal-affine}
Suppose $U = \cbr{ \theta \in \RR^d: X \theta = y  }$ is an nonempty affine subspace of $\RR^d$, where $X \in \RR^{m \times d}$ has rows $x_1, \ldots, x_m \in \RR^d$. 
Let $\dim(\spn(x_1, \ldots, x_m)) = l$, and 
let $W \in \RR^{d \times (d-l)}$ be a matrix whose columns form an orthonormal basis of $\spn(x_1, \ldots, x_m)^\perp$. 
Consider $Z \sim \N(0,I_d)$; then,
\[
Z \mid \cbr{Z \in U} \sim \N( X^\dagger y, W W^\top ).
\]
\end{lemma}
\begin{proof}
Denote by $\hat{\theta} = X^\dagger y$ the least norm solution of equation $X\theta = y$. It is well-known that $\hat{\theta} \in \spn(x_1, \ldots, x_m)$. 
As $U \neq \emptyset$, $X \hat{\theta} = y$.
We now claim that $U$ can be equivalently written as $\cbr{ \hat{\theta} + W \alpha: \alpha \in \RR^{d-l} }$:
\begin{enumerate}
\item On one hand, for all $\theta = \hat{\theta} + W\alpha$, $X \theta = X \hat{\theta} + X W\alpha = y + 0 = y$. 

\item On the other hand, for every $\theta \in U$, as $X \theta = y$, we have $X (\theta - \hat{\theta}) = \vec{0}$, which implies that $\theta - \hat{\theta} \in \spn(x_1, \ldots, x_m)^\perp$. Therefore, there exists some $\alpha \in \RR^{d-l}$ such that $\theta = \hat{\theta} + W\alpha$.
\end{enumerate}

Define $V \in \RR^{d \times l}$ to be a matrix whose columns form an orthonormal basis of $\spn(x_1, \ldots, x_m)$. We also claim that given a vector $z \in \RR^d$, $z \in U \Leftrightarrow V^\top z = V^\top \hat{\theta}$:
\begin{enumerate}
    \item If $z \in U$, by the previous claim, $z = \hat{\theta} + W \alpha$, and therefore $V^\top z = V^\top \hat{\theta} + V^\top W \alpha = V^\top \hat{\theta}$.
    \item If $V^\top z = V^\top \hat{\theta}$, then note that $z = V V^\top z + W W^\top z = V V^\top \hat{\theta} + W (W^\top z) = \hat{\theta} + W (W^\top z) $, where the last equality follows from that $\hat{\theta} \in \spn(x_1, \ldots, x_m)$. Taking $\alpha_z = W^\top z \in \RR^{d-l}$, we have $z = \hat{\theta} + W \alpha_z$, implying that $z \in U$.
\end{enumerate}

For the rest of the proof, let $\stackrel{d}{=}$ denote equality in distribution. 
Consider random variable $Z \stackrel{d}{=} \N(0,I_d)$. 
Let $\epsilon_V = V^\top Z, \epsilon_W = W^\top Z$.
Now, note that the matrix 
$T = \begin{pmatrix} W^\top \\ V^\top \end{pmatrix} \in \RR^{d \times d}$ is a  orthonormal matrix,
\[
\begin{pmatrix} \epsilon_V \\ \epsilon_W \end{pmatrix}
=
\begin{pmatrix} V^\top \\ W^\top \end{pmatrix} Z = TZ 
\stackrel{d}{=} 
\N(0,I_d),
\]
Therefore, $\epsilon_V$, $\epsilon_W$ are two independent, standard normal random variables with distributions $\N(0, I_l)$ and $\N(0, I_{d-l})$, respectively.

Note from the second claim that the event $\cbr{Z \in U}$ is equivalent to $\{ \epsilon_V = V^\top \hat{\theta} \}$; therefore, $\epsilon_W \mid \cbr{Z \in U} \stackrel{d}{=} \N(0, I_{d-l})$. 
As a result, 
\[
Z \mid \cbr{Z \in U}
\stackrel{d}{=}
V \epsilon_V + W \epsilon_W \mid \cbr{Z \in U}
\stackrel{d}{=}
\hat{\theta} + W \epsilon_W \mid \cbr{Z \in U}
\stackrel{d}{=}
\N( X^\dagger y, W W^\top ).
\qedhere
\] 
\end{proof}

\section{Experiments}
\label{sec:experiments}

In this section, we empirically explore the shrinkage of the version space under various baseline methods and Algorithm~\ref{alg:efficient-audit}. The two baseline methods of sampling we will consider are: 1) i.i.d sampling (without replacement) 2) active learning (CAL).

\textbf{Procedure:} We train a logistic regression model to find $h^*$ on two datasets commonly used in Fairness literature. The first is COMPAS~\cite{compas}, where the two groups are defined to be Caucasian and non-Caucasian. And the second is the Student Performance Dataset, where the two groups are defined to be Female and Male. Then, we run the three methods with an alloted label budget of: $20, 50, 80, 100, 120$. These are a small fraction of the total dataset size (much smaller for COMPAS than Student Performance).

\textbf{Evaluation:} Our evaluation will be on the version space induced by the labels requested by the three methods. We will evaluate the version space in two ways:

\begin{enumerate}
    \item Given $\Hcal[S]$, we will compute its $\mu$-diameter $\max_{h, h' \in \Hcal[S]} \mu(h) - \mu(h')$. The $\mu$-diameter of the version space captures the largest extent that the algorithm's $\mu$ estimate may be changed by post-hoc manipulation. The smaller it is the higher the degree of manipulation-proofness.
    
    To compute $\max_{h, h' \in \Hcal[S]} \mu(h) - \mu(h')$ , we will evaluate $\max_{h \in \Hcal[S]} \mu(h)$ and $\min_{h \in \Hcal[S]} \mu(h)$. Let $G_1 = \{x \in \Xcal: x_A = 1\}$ and $G_0 = \{x \in \Xcal: x_A = 0\}$. To implement the maximization program, we may move the constraint into the objective as a Lagrangian:

    $$\max_{h} \frac{1}{|G_1|} \sum_{x \in G_1} \ind\{h(x) = 1\} - \frac{1}{|G_0|} \sum_{x \in G_0} \ind\{h(x) = 1\} + \lambda(\sum_{x \in S} \ind\{h(x) = h^*(x)\})$$

    or equivalently:

    $$\max_{h} \frac{1}{|G_1|} \sum_{x \in G_1} \ind\{h(x) = 1\} + \frac{1}{|G_0|} \sum_{x \in G_0} \ind\{h(x) = -1\} + \lambda(\sum_{x \in S} \ind\{h(x) = h^*(x)\})$$

    As mentioned earlier, we observe that this objective may be framed as a cost-sensitive classification problem, which is commonly used in fairness literature~\cite{agarwal2018reductions}. In particular, the cost for predicting $1$ for $x \in G_1$ is $-\frac{1}{|G_1|}$ and $0$ o.w, the cost for predicting $1$ is $0$ for $x \in G_0$ and $-\frac{1}{|G_0|}$ o.w and the cost for predicting $h^*(x)$ for $x \in S$ is $-\lambda$ and $0$ o.w.
    By using iterative doubling and grid search, we look for the smallest $\lambda$ such that we may enforce $h(x) = h^*(x)$ $\forall x \in S$ (since these hard constraints) and find the maximizing $h$ in the version space given this $\lambda$. The same procedure is applied for the minimizing $h$ in the version space.
    
    \item Since we may choose any $\mu(h)$ for $h \in \Hcal[S]$ to return as an estimate for $\mu(h^*)$, we will evaluate $\EE_{h \sim \text{unif}(\Hcal[S])}[|\mu(h) - \mu(h^*)|]$ -- this corresponds to the average error and is proportional to estimation accuracy.

    For sampling from the version space, we will use the classic hit-and-run algorithm and sample $500$ models from the version space at each budget and then average the error.
    
\end{enumerate}

\textbf{Results:} In terms of the $\mu$-diameter of the version space, which may be interpreted as the maximum possible degree of post-audit manipulation of $\mu$, we see in Figure~\ref{fig:diam_comparison} that Algorithm~\ref{alg:efficient-audit} is the best of the three methods at all budgets. This is expected since Algorithm~\ref{alg:efficient-audit} is designed to make use of $\max_{h \in \Hcal[S]} \mu(h)$ and $\min_{h \in \Hcal[S]} \mu(h)$ estimates in its query selection to ``shrink'' the version space in $\mu$-space. Behind Algorithm~\ref{alg:efficient-audit}, CAL looks to be generally better or on-par with i.i.d sampling. 

In terms of estimation error, going by the average $\mu$ estimation error in the version space, we see in Figure~\ref{fig:est_comparison} that in general, one of the active approaches outperforms that of i.i.d sampling. Between the two active approaches, there are budgets setting where one is better than the other and vice versa.

\begin{figure}[htb]
    \begin{minipage}[t]{.45\textwidth}
        \centering
        \includegraphics[width=\textwidth]{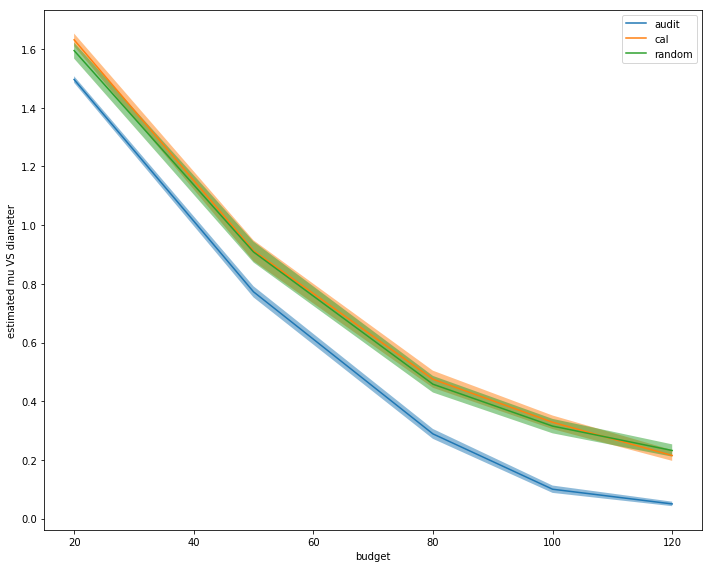}
    \end{minipage}
    \hfill
    \begin{minipage}[t]{.45\textwidth}
        \centering
        \includegraphics[width=\textwidth]{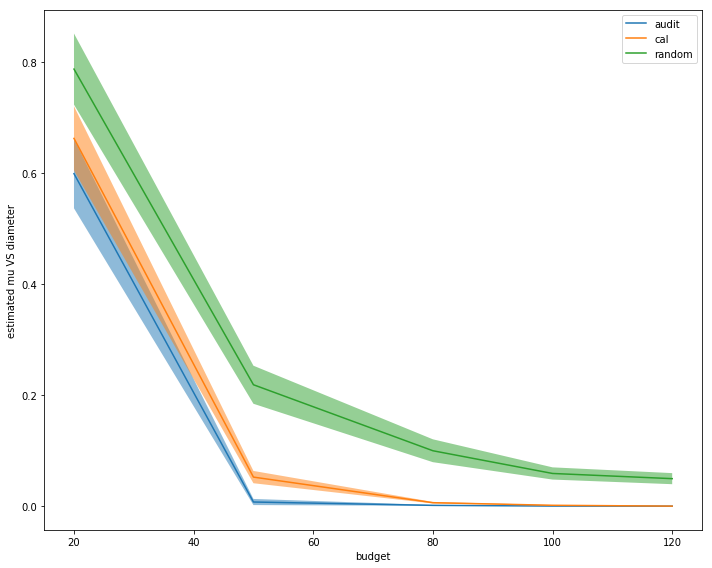}
    \end{minipage} 
    \caption{Left: Comparison of the three methods on the Student Performance dataset on $\mu$-diameters of the final version spaces, as a function of label query budget. Right: Comparison of the three methods on the COMPAS dataset. For the error bars, a $95$ percent confidence interval is constructed using the $50$ repeats at each budget.}
    \label{fig:diam_comparison}
\end{figure}

\begin{figure}[htb]
    \begin{minipage}[t]{.45\textwidth}
        \centering
        \includegraphics[width=\textwidth]{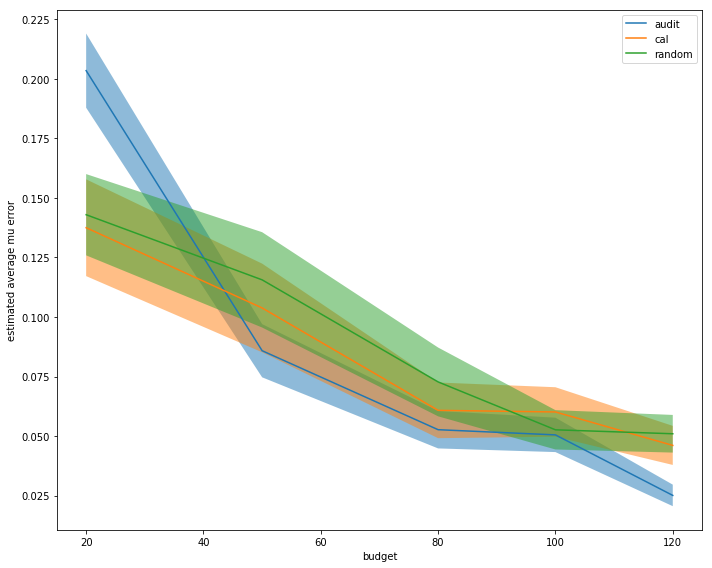}
    \end{minipage}
    \hfill
    \begin{minipage}[t]{.45\textwidth}
        \centering
        \includegraphics[width=\textwidth]{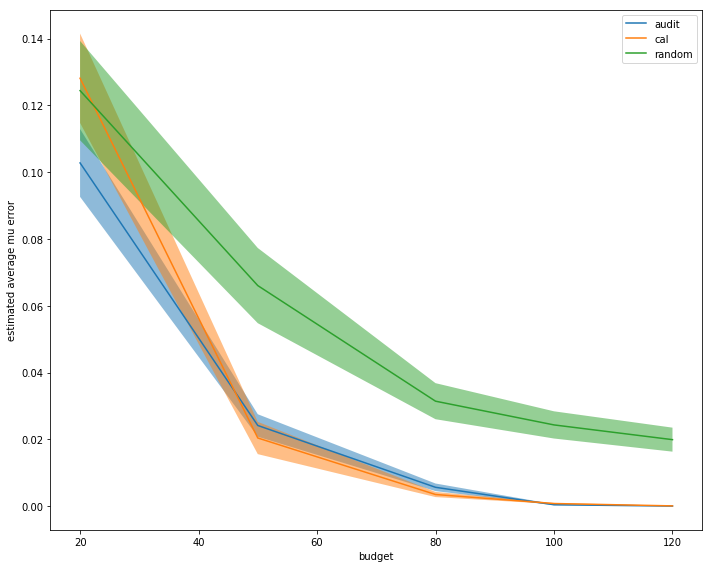}
    \end{minipage} 
    \caption{Left: Comparison of the three methods on the Student Performance dataset on average $\mu$-estimation errors of the final version spaces, as a function of label query budget. Right: Comparison of the three methods on the COMPAS dataset. For the error bars, a $95$ percent confidence interval is constructed using the $50$ repeats at each budget.}
    \label{fig:est_comparison}
\end{figure}

\end{document}